%% file: main.tex
\documentclass{article}

\input{preamble}

\title{Kernel Tests of Equivalence}

\date{} 					% Or removing it

\author{%
Xing Liu\thanks{Corresponding email: \texttt{xingliu97@outlook.com}} \\
QuantCo \\
London, UK \\
\And
Axel Gandy \\
Department of Mathematics \\
Imperial College London \\
London, UK \\
}

\renewcommand{\shorttitle}{Kernel Tests of Equivalence}

\begin{document}
\maketitle

\begin{abstract}
    We propose novel kernel-based tests for assessing the equivalence between distributions. Traditional goodness-of-fit testing is inappropriate for concluding the absence of distributional differences, because failure to reject the null hypothesis may simply be a result of lack of test power, also known as the Type-II error. This motivates \emph{equivalence testing}, which aims to assess the \emph{absence} of a statistically meaningful effect under controlled error rates. However, existing equivalence tests are either limited to parametric distributions or focus only on specific moments rather than the full distribution. We address these limitations using two kernel-based statistical discrepancies: the \emph{kernel Stein discrepancy} and the \emph{Maximum Mean Discrepancy}. The null hypothesis of our proposed tests assumes the candidate distribution differs from the nominal distribution by at least a pre-defined margin, which is measured by these discrepancies. We propose two approaches for computing the critical values of the tests, one using an asymptotic normality approximation, and another based on bootstrapping. Numerical experiments are conducted to assess the performance of these tests.
\end{abstract}

% keywords can be added

\section{Introduction}
Goodness-of-fit (GOF) testing aims to use observed data to draw statistical conclusions on how well some nominal distribution approximates an unknown underlying data-generating process. Denote by $\cP$ the set of probability measures on $\R^d$, and let $Q, P \in \cP$. Given observations $\X_n \coloneqq \{ X_i \}_{i=1}^n \subset \R^d$ drawn independently from $Q$, a GOF test aims to use $\X_n$ to test
\begin{align}
    H_0^\ast: \ Q = P
    \qquad \textrm{against}\qquad
    H_1^\ast: \ Q \neq P
    \;.
    \label{eq:standard_hyps}
\end{align}
Depending on the assumed knowledge about the nominal distribution $P$, the GOF testing literature can be categorized into two main settings. The first is \emph{one-sample} testing, where samples from $Q$ are observed, and $P$ is often specified via a probability density function, whose normalizing constant is usually unknown. This is the case when, e.g., $P$ belongs to a probabilistic model, such as general exponential family models \citep[Chapter 3.4]{casella2002statistical}, energy-based models \citep{cho2013gaussian} and probabilistic graphical models \citep{koller2009probabilistic}. Another setup is \emph{two-sample} testing, where $P$ is specified implicitly through a sampling procedure and does not admit tractable likelihood functions, but simulating from $P$ is possible. Examples include \emph{generative adversarial networks} \citep{goodfellow2014generative} and many scientific models based on simulators \citep{beaumont2010approximate,riesselman2018deep,bharti2022general}. In both settings, the null hypothesis $H_0^\ast$ is rejected if there is statistically significant evidence that the data do not comply to the nominal distribution $P$.

In many applications, however, the goal is not to establish that a statistically significant difference exists between some distributions, but rather that the distributions are practically equivalent. Examples include comparative bioequivalence trials \citep{metzler1974bioavailability,schuirmann1987comparison,cade2011estimating}, drug stability assessment \citep{liu2007pooling}, pharmacokinetics studies  \citep{hauck1984new,gsteiger2011simultaneous}, and validation of statistical models \citep{dette1998validation,dette2018equivalence,carlini2017adversarial}. In those cases, hypotheses of the form \eqref{eq:standard_hyps} are no longer appropriate for two reasons. First, failure to reject the null hypothesis $H_0^\ast$ does \emph{not} provide a probabilistic guarantee for $Q = P$; instead, it can simply be due to the lack of test power \citep{inman1994karl,siebert2020validation}. Second, it is well-known that, in practice, the null hypothesis $H_0^\ast$ will eventually always be rejected as the data size grows. This is because ``all models are wrong'', so a very large sample would ``invariably produce statistically significant lack of fit'' \citep[pp.\ 189]{hardle2019applied}.

\emph{Equivalence testing} \citep{david2013testing} aims to resolve this issue by considering a different null hypothesis, where $Q$ and $P$ are sufficiently \emph{different} by a pre-specified margin. The notion of distributional difference is often quantified by a statistical discrepancy $D$. Given a positive number $\theta > 0$, known as the \emph{equivalence margin} \citep{wellek2021testing} or \emph{minimal meaningful distance} \citep{meyners2012equivalence}, we consider testing the following hypotheses
\begin{align}
    H_0: \ D(Q, P) \geq \theta
    \;\qquad\textrm{against}\qquad
    H_1: \ D(Q, P) < \theta
    \;.
    \label{eq:equiv_hyps}
\end{align}
Unlike the standard hypotheses \eqref{eq:standard_hyps}, rejection of $H_0$ now suggests that there is sufficient evidence that $Q$ and $P$ are similar. This allows one to rigorously control the probability of falsely concluding that $Q$ and $P$ are equivalent.

Equivalence testing was initially employed in the field of pharmacokinetics \citep{hauck1984new} for determining whether new, cheaper drugs exhibit comparable efficacy to existing, more expensive ones \citep{senn2008statistical}. For a review, see \citet{david2013testing,lakens2017equivalence}. In the statistics literature, equivalence tests have been developed for comparing parametric models \citep{dette1998validation,dette2018equivalence}, for assessing independence \citep{chen2023testing}, and for improving the test-then-pool framework in causal inference \citep{li2020revisit}. However, most existing methods are restricted to parametric models or specific moments. A recent non-parametric method by \citet{chen2023testing} uses the Maximum Mean Discrepancy \citep{muller1997integral,gretton2012kernel} to avoid model-specific assumptions. However, their method relies on an asymptotic normal approximation, which, as we show numerically in \Cref{sec:one_sample:bootstrap_test}, may break down when the equivalence margin $\theta$ is small, resulting in uncontrolled Type-I error.

We propose two families of equivalence tests, one based on an asymptotic normal approximation as in \citet{chen2023testing}, and the other on a bootstrapping technique. Each family contains two variants, one employing Kernel Stein Discrepancies, suitable for one-sample testing, and the other employing Maximum Mean Discrepancies, suitable for two-sample testing. Our tests do not make any parametric assumptions about the underlying distributions and are not limited to specific moments. Our contributions are summarized as follows and in \Cref{tab:summary}:
\begin{enumerate}
    \item We first propose two equivalence tests for the one-sample setting. The first test, called the \emph{E-KSD-Normal} test, is motivated by \citet{chen2023testing} and leverages a Central Limit Theorem for the KSD statistic. We then demonstrate that this test, despite achieving high power, can result in uncontrolled Type-I error, especially with small equivalence margins. We then propose an alternative equivalence test, called the \emph{E-KSD-Boot} test, which is based on bootstrapping and can achieve better Type-I error control with finite samples regardless of the size of the equivalence margin, albeit at the cost of a lower test power.
    \item Next, we propose two equivalence tests for the two-sample setting. These tests, called the \emph{E-MMD-Normal} and \emph{E-MMD-Boot} tests, respectively, mirror the aforementioned one-sample tests and leverage MMD to adapt them to the two-sample setting. In particular, our E-MMD-Normal test resembles the homogeneity-equivalence test of \citet{chen2023testing} and generalizes it to allow non-identical sample sizes.
    \item For the two bootstrapping tests, we propose a data-driven approach for selecting the equivalence margin $\theta$, which computes $\theta$ as the minimal effect size given a pre-specified test power.
\end{enumerate}

The rest of this work is organized as follows. \Cref{sec:background} reviews the equivalence testing framework and the two statistical discrepancies we will use to construct our tests. \Cref{sec:one_sample} and \Cref{sec:two_sample} introduce our proposed equivalence tests that are suitable for one-sample and two-sample testing, respectively. \Cref{sec:choice_of_equiv_margin} discusses how to select the equivalence margin in practice. \Cref{sec:experiments} provides numerical experiments showing the validity and power of the proposed tests.

\begin{table}[t]
    \begin{center}
        % \resizebox{\textwidth}{!}{
        \begin{tabular}{lcccc}
            \toprule
            \multicolumn{1}{c}{\bf Equivalence Test} & \multicolumn{1}{c}{\bf Discrepancy} & \multicolumn{1}{c}{\bf Setting} & \multicolumn{1}{c}{\bf Method} & \multicolumn{1}{c}{\bf Reference}
            \\
            \midrule
            E-KSD-Normal & KSD & One-sample & CLT & \Cref{thm:ksd_normal_test}
            \\
            E-KSD-Boot & KSD & One-sample & Bootstrapping & \Cref{thm:ksd_bootstrap_test}
            \\
            E-MMD-Normal & MMD & Two-sample & CLT & \citet{chen2023testing} and \Cref{thm:mmd_normal_test}
            \\
            E-MMD-Boot & MMD & Two-sample & Bootstrapping & \Cref{thm:mmd_bootstrap}
            \\
            \bottomrule
        \end{tabular}
        % }
    \end{center}
    \caption{Summary of kernel-based equivalence tests proposed in this paper, and their connections to existing literature.}
    \label{tab:summary}
\end{table}

\section{Background}
\label{sec:background}
This section provides the necessary background on the kernel-based divergences and metrics that we adopt in our tests, as well as the relevant literature. Throughout, we denote by $\cP(\R^d)$ the set of probability measures on $\R^d$.

\subsection{Equivalence Testing}

Our goal is to design non-parametric tests targeting the equivalence hypotheses of the form \eqref{eq:equiv_hyps}, with appropriate choices of $D$ to be introduced later. Tests targeting equivalence hypotheses of the form \eqref{eq:equiv_hyps} are known as \emph{equivalence testing} \citep[ET,][]{david2013testing,lakens2017equivalence} or \emph{negligible effect testing} \citep{beribisky2024evaluating}. In classic GOF testing, the null hypothesis typically assumes no meaningful statistical difference, namely $H_0^\ast: Q = P$. Consequently, GOF testing can only conclude the \emph{presence} of model mis-alignment, because failing to reject $H_0^\ast$ does not suggest any statistical evidence that it holds true. In contrast, for ET, the null hypothesis assumes instead that $Q$ and $P$ are sufficiently different by some margin. This means that, for a well-calibrated ET, the probability of falsely concluding the \emph{absence} of meaningful statistical differences can be controlled.

ET has found applications in a wide range of areas, including pharmacokinetics \citep{hauck1984new,schuirmann1987comparison} and psychology \citep{quertemont2011statistically,rogers1993using}. The simplest equivalence test is the \emph{two one-sided test} \citep[TOST,][]{schuirmann1987comparison,seaman1998equivalence}. In TOST, two one-sided hypotheses of the form $H_0: D > \theta_u$ and $H_0: \Delta < - \theta_l$ are tested, where $D$ is some measure of effect (potentially negatively valued) and $\theta_l,\theta_u$ are two thresholds for significant effect size. One then one conclude $-\theta_l < \Delta < \theta_u$ if both hypotheses are rejected. However, TOST is often found to suffer from low test power due to the use of two auxiliary hypotheses \citep{boulaguiem2024finite,neuhauser2024perspective}.

Various ETs targeting directly an equivalence hypothesis have been proposed. \citet{rogers1993using,wellek2002testing} proposed tests for equivalence of means. \citet{dette2021bio} studied the problem of testing equivalence of means and variances in a two-sample setup. \citet{dette1998validation} used the $L_2$ distance to test the equivalence between observations and parametric models. \citet{dette2018equivalence} used both $L_2$ and $L_\infty$ distances to test the equivalence of two parametric models. \citet{baringhaus2017limit} used a weighted $L_2$ distance. \citet{freitag2007nonparametric} proposed an equivalence test for comparing the marginals of multivariate distributions. \citet{yuan2016assessing,beribisky2024evaluating} studied equivalence testing for structural equation models. These tests are limited to either specific moments of the distributions, or to  certain parametric models. In contrast, our proposed tests are non-parametric and capable of testing for the homogeneity of the full distributions, thus providing more flexibility.

Closest to our work is \citet{chen2023testing}, who used the characteristic functions to construct tests for distributional homeogeneity, independence, and symmetry. Although not mentioned explicitly in their paper, their measure of dissimilarity is effectively an MMD. Our normality-based methods are similar to their approach, except that they assume the two samples to have equal sizes, whereas we allow them to be different, thus providing more generality. Moreover, our bootstrap-based methods differ fundamentally from their approach---while \citet{chen2023testing} rely on an asymptotic normality result for the MMD estimator, our bootstrapped tests leverage the triangle inequality of MMD or KSD to construct a test statistic. As we will show numerically in \Cref{sec:experiments}, with small sample sizes, the normality-based tests, including that of \citet{chen2023testing}, performs poorly in controlling the Type-I error when the sample size is small, whereas the bootstrap-based tests remain well-calibrated even with small samples.

\begin{figure}
    \centering
    \includegraphics[width=0.75\linewidth]{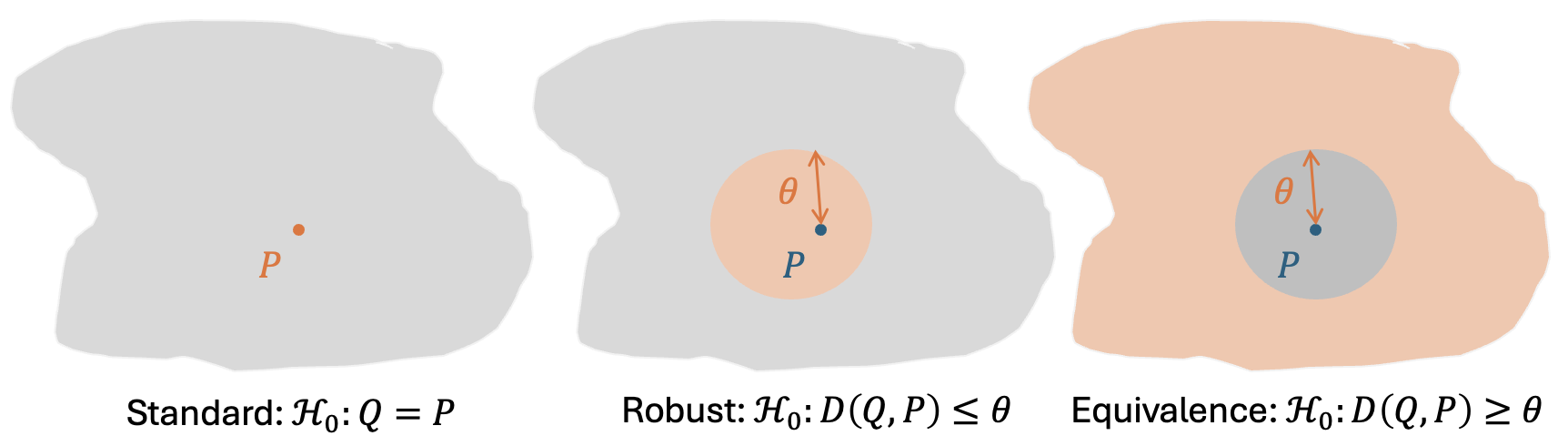}
    \caption{Comparison of different hypothesis testing paradigms. The shaded area represents the space of distributions of interest. The null sets are shown in orange, and the alternative sets in grey. \emph{Left.} standard testing with a point null hypothesis. \emph{Middle.} Robust testing based on a statistical discrepancy $D$. \emph{Right.} Equivalence testing.}
    \label{fig:hypotheses}
\end{figure}

\subsection{Kernel-based Discrepancies}

The equivalence hypotheses \eqref{eq:equiv_hyps} are specified by a statistical discrepancy $D$. Hence, changing $D$ would result in a different null set, thus also a different hypothesis. It is therefore important to select a statistical discrepancy that is both computationally convenient for the problem at hand, and also statistically meaningful to quantify the type of model deviation of interest. We propose to use KSD and MMD, two kernel-based statistical discrepancies with desirable computationally and statistical properties. We will review KSD and MMD in this section. Throughout, we assume the kernel $k$ is \emph{characteristic} \citep{sriperumbudur2011universality}.

\subsubsection{Kernel Stein Discrepancy}
\label{sec:ksd}
KSD is a statistical divergence based on kernel methods and Stein's method. Assume $P$ admits a Lebesgue density $p$ on $\R^d$ and $p(x) > 0$ that is continuously differentiable, and define its \emph{score function} $s_p(x) \coloneqq \nabla_x \log(x)$. Let $\cH_k$ be the reproducing kernel Hilbert space \citep[RKHS;][]{saitoh2016theory} associated with a reproducing kernel $k$, and let $\cH_k^d = \{f=(f_1, \ldots, f_d)^\top: f_j \in \cH\}$ be the product space equipped with the norm $\| f \|_{\cH_k^d} = (\sum_{j=1}^d \|f_j\|_{\cH_k}^2 )^{1/2}$. The \emph{(Langevin) kernel Stein discrepancy} \citep[KSD;][]{chwialkowski2016kernel,liu2016kernelized,oates2017control} between $Q$ and $P$ is defined as
\begin{align*}
    \KSD(Q, P)
    \;&=\;
    \sup_{f \in\cH_k^d; \| f \|_k \leq 1} \big| \E_{X \sim Q}[\cA_p f(X)] \big|
    \;,
\end{align*}
where $\cA_p$ is an operator acting on vector-valued functions, defined as $(\cA_p f)(x) = s_p(x)^\top f(x) + \nabla^\top f(x)$. Under mild regularity conditions on $s_p$, the operator $\cA_p$, called \emph{Langevin Stein operator}, maps sufficiently regular functions to zero-mean functions under the probability measure $P$, namely $\E_{X \sim Q}[\cA_p f(X)] = 0$. Loosely speaking, this means that the KSD takes small values when $Q$ is ``close'' to $P$, and large values otherwise. This can be made rigorous by showing that KSD is \emph{$P$-separating}, namely $\KSD(Q, P) = 0$ if $Q = P$ and $\KSD(Q, P) > 0$ otherwise, whenever the kernel $k$ is characteristic \citep[Theorem 3]{barp2024targeted}. This property suggests that KSD is a valid statistical discrepancy, motivating \citet{chwialkowski2016kernel,liu2016kernelized} to use it as a test statistic for the standard GOF null hypothesis \eqref{eq:standard_hyps}, where $H_0^\ast$ is rejected for large estimates of $\KSD(Q, P)$.

When $\E_{X \sim P}[\| s_p(X) \|_2^2] < \infty$, the \emph{squared} KSD admits a double-expectation form, allowing it to be efficiently estimated by Monte Carlo. Indeed, \citet[Corollary 1]{barp2024targeted} shows that
\begin{align*}
    \KSD(Q, P)
    \;&=\;
    \big( \E_{X, X' \sim Q}[u_p(X, X')] \big)^{1/2}
    \;,
\end{align*}
where 
\begin{align*}
    u_p(x, x')
    \;&=\;
    s_p(x)^\top s_p(x') k(x, x')
    + s_p(x)^\top \nabla_2 k(x, x')
    + s_p(x')^\top \nabla_1 k(x, x')
    + \nabla^\top \nabla k(x, x')
    \;,
\end{align*}
and $\nabla^\top \nabla k(x,y) \coloneqq \sum_{j=1}^d \partial_{x^j}\partial_{y^j} k(x, y)$. Let $Q_n$ be an empirical distribution formed by independent random variables $X_1, \ldots, X_n$. a natural estimator is the following V-statistic, formed by substituting $Q_n$ for $Q$
\begin{align}
    \KSD^2(Q_n, P)
    \;=\;
    \frac{1}{n^2} \sum_{1 \leq i, j \leq n} u_p(X_i, X_j)
    \;.
    \label{eq:ksd_vstat}
\end{align}
This estimator possesses desirable computational advantages. Firstly, it does not require samples from $P$, which can be difficult to generate. Secondly, it depends on $P$ only through the score function $s_p$, which can be evaluated even if the normalizing constant is unknown. These make KSD suitable for testing the GOF of many probabilistic models that are beyond the reach of traditional GOF tests, such as general exponential family models \citep{barndorff-nielsen1978information,canu2006kernel}, energy-based models \citep{cho2013gaussian}, and probabilistic graphical models \citep{koller2009probabilistic}. In comparison, classic GOF tests, such as likelihood ratio tests \citep[Chapter 8.7]{hogg1977probability}, chi-square tests \citep[Chapter 9.1]{hogg1977probability}, Komogorov-Smirnov tests \citep{kolmogorov1933sulla}, either require the normalizing constant to be known, or can struggle for multivariate models. 

GOF tests based on KSD proceed by first using empirical samples to compute the KSD estimate, and rejecting the null hypothesis \eqref{eq:standard_hyps} if it exceeds a critical value. This critical value does not admit a closed-form, but can be estimated using \emph{wild bootstrapping} \citep{leucht2013dependent,shao2010dependent} or \emph{weighted bootstrapping} \citep{arcones1992bootstrap,janssen1994weighted} \citep{liu2025robustness}. Both approaches are found to achieve similar empirical performance for the KSD test \citep{liu2025robustness}. In our proposed tests, we will use weighted bootstrapping since our theoretical results are built upon their existing theory studied in \citet{arcones1992bootstrap}.

In weighted bootstrapping, the distribution of $\KSD^2(Q_n, P)$ is approximated using bootstrap samples of the form
\begin{align}
    D^2_{W_n}(\X_n; u_p)
    \;\coloneqq\;
    \frac{1}{n^2} \sum_{1 \leq i, j \leq n} (W_{ni} - 1) (W_{nj} - 1) u_p(x_i, x_j)
    \;,
    \label{eq:ksd_bootstrap_sample}
\end{align}
where $W_n = (W_{n1}, \ldots, W_{nn}) \sim \textrm{Multinomial}(n; 1/n, \ldots, 1/n)$. Let $\{\KSD_{n,b}^2\}_{b=1}^B$ denote the bootstrap sample \eqref{eq:ksd_bootstrap_sample} based on $B$ i.i.d.\ copies of $W_n$, and let $\{\KSD_{n,b}\}_{b=1}^B$ be their square-roots. The standard KSD test rejects $H_0^\ast$ if $\KSD^2(Q_n, P) > \gamma_{1-\alpha}^B$, where $\gamma^B_{1-\alpha}$ is the $(1-\alpha)$-th quantile of $\{ \KSD_{n,b} \}_{b=1}^B$, namely
\begin{align}
    \gamma^B_{1-\alpha}
    \;=\;
    \inf\bigg\{u \in \R: \; \frac{1}{B} \sum_{b = 1}^{B} \indicator\{ \KSD_{n,b} \leq u \} \geq 1 - \alpha \bigg\}
    \;.
    \label{eq:ksd_bootstrap_quantile}
\end{align}
Beyond the original KSD test of \citet{liu2016kernelized,chwialkowski2016kernel}, various extensions have been proposed to specialize in different model families, such as discrete distributions \citep{xu2020stein}, survival models \cite{fernandez2020kernelized}, graphical models \citep{xu2021stein}, and biological sequences \citep{amin2023kernelized}. However, existing methods focus primarily on testing the standard point null hypothesis \eqref{eq:standard_hyps}. As a result, these tests are only suitable for detecting the presence of model deviation, but not for concluding model equivalence. \citet{liu2025robustness} proposed a \emph{robust KSD test} for null hypotheses of the form $H_0: D(Q, P) \leq \theta$, which accounts for mild model mis-specification, but this is still not suitable for testing model equivalence.

% Common approaches for computing the critical value includes permutation and bootstrapping \citep{gretton2012kernel,schrab2023mmd}. In the hypothesis testing literature, MMD has been studied extensively both with standard point null hypotheses \citep{gretton2012kernel,lloyd2015statistical,shekhar2022permutation,schrab2023mmd} and with robust hypotheses \citep{sun2023kernel,key2025composite,schrab2024robust}. 

\subsubsection{Maximum Mean Discrepancy}
\label{sec:mmd}

Another closely related family of kernel discrepancies is the MMD. The MMD between two probability measures $Q, P \in \cP$ is defined as the maximal discrepancy between the means of functions under $Q$ and $P$ over all functions in the unit ball of $\cH$, i.e.,
\begin{align}
    \MMD(Q, P)
    \;\coloneqq\;
    \sup_{f \in \cH: \; \| f \|_{\cH} \leq 1} \big| \E_Q[f] - \E_P[f] \big|
    \label{eq: MMD}
    \;.
\end{align}
With bounded kernels, the MMD can be rewritten as a distance between the \emph{kernel mean embeddings} of two distributions \citep[Lemma 4]{gretton2012kernel}. The kernel mean embedding \citep{smola2007hilbert,sriperumbudur2010hilbert} is a mapping from $\cP$ to $\cH$, denoted as $\mu_P(\cdot) \coloneqq \int k(x, \cdot) P(\diff x)$ for any $P \in \cP$. One can show that \citep[Theorem 5]{gretton2012kernel} the MMD can be expressed in closed-form as
\begin{align}
    \MMD(Q, P)
    \;=\;
    \| \mu_P - \mu_Q \|_{\cH}
    \;=\;
    \big(
        \E_{X, X' \sim Q}[k(X, X')] + \E_{Y, Y' \sim P}[k(Y, Y')]
        - 2 \E_{X \sim Q, Y \sim P}[k(X, Y)]
    \big)^{1/2}
    \;.
    \label{eq:mmd_closed_form}
\end{align}
When $k$ is characteristic, the kernel mean embedding is injective, and the MMD is a metric on $\cP$ \citep{sriperumbudur2016optimal,simon2023metrizing}. Many well-known kernels are characteristic, such as Gaussian, Inverse Multi-Quadric and Mat\'{e}rn kernels, while an example of non-characteristic kernels is linear kernels.

Equation \eqref{eq:mmd_closed_form} suggests a natural estimator for the \emph{squared} MMD \citep[Eq.~5]{gretton2012kernel}. Given independent random samples $\X_n = \{ X_i \}_{i=1}^n$ from $Q$ and $\Y_m = \{ Y_j \}_{j=1}^m$ from $P$, the squared MMD, $\MMD^2(Q, P)$, can be estimated by substituting the empirical measures $Q_n = n^{-1} \sum_{i=1}^n \delta_{X_i}$ and $P_m = m^{-1} \sum_{j=1}^m \delta_{Y_j}$, i.e.,
\begin{align}
    \MMD^2(Q_n, P_m)
    \;=\;
    % \Big(
        \frac{1}{n^2}\sum_{1 \leq i, j \leq n} k(X_i, X_j)
        + \frac{1}{m^2}\sum_{1 \leq i, j \leq m} k(Y_i, Y_j)
        - \frac{2}{nm}\sum_{i=1}^n \sum_{j=1}^m k(X_i, Y_j)
    % \Big)^{1/2}
    \label{eq:mmd_v_stat}
    \;.
\end{align}
The above estimator $\MMD(Q_n, P_m)^2$ is biased but consistent \citep{gretton2012kernel}. Moreover, \eqref{eq:mmd_v_stat} can be equivalently expressed as a \emph{two-sample V-statistic} \citep{hoeffding1948class,kim2022minimax} of degree $(2,2)$
\begin{align}
    \MMD^2(Q_n, P_m)
    \;=\;
    \frac{1}{n^2m^2}\sum_{1 \leq i, i' \leq n} \sum_{1 \leq j, j' \leq m} h(X_i, X_{i'}, Y_j, Y_{j'}) \;,
    \label{eq:mmd_v_stat_two_sample}
\end{align}
where $
h(x, x', y, y') = k(x, x') + k(y, y') - k(x, y') - k(x', y)
$. 
Expressing the MMD estimator as \eqref{eq:mmd_v_stat_two_sample} offers theoretical advantages, as it allows us to use well-established results on V-statistics to construct our tests. In practice, however, \eqref{eq:mmd_v_stat} should be used for computation, since it is more efficient with cost $\cO((n + m)^2)$. In this work, we focus on the general case where $n,m$ can be different. In the special case when $n = m$, the estimator \eqref{eq:mmd_v_stat} can be simplified to a \emph{one-sample} V-statistic \citep{gretton2012kernel,lloyd2015statistical,ramdas2015decreasing}.

The MMD test rejects the standard null hypothesis \eqref{eq:standard_hyps} if the MMD estimate exceeds a critical value. Common approaches for computing the critical value includes permutation and wild bootstrapping \citep{gretton2012kernel,schrab2023mmd}. In the hypothesis testing literature, MMD has been studied extensively both with standard point null hypotheses \citep{gretton2012kernel,lloyd2015statistical,shekhar2022permutation,schrab2023mmd} and with robust hypotheses \citep{sun2023kernel,key2025composite,schrab2024robust}.

\subsection{Other Related Work}
Equivalence testing is closely connected with \emph{robust testing} \citep{fauss2021minimax}, which often targets hypotheses of the form $H_0: D(Q, P) \leq \theta$ against $H_1: D(Q, P) > \theta$. That is, their null and alternative hypotheses are interchanged compared with \eqref{eq:equiv_hyps}; see \Cref{fig:hypotheses} for an illustration. Robust tests based on both MMD and KSD have been studied \citep{sun2023kernel,liu2025robustness}. Compared with standard GOF where the null hypothesis is a singleton, robust tests hypothesize that the model is \emph{close enough}, rather than identical, to the data-generating distribution. This allows one to ignore mild model deviations, so that a model would not be rejected only due to minor lack-of-fit effects that are not scientifically interesting. However, since the null set in a robust test still assumes that $Q$ and $P$ are close, they are also not appropriate for concluding the absence of model deviation.

% the null set is typically a neighbourhood around some nominal distribution. Such neighbourhood is often constructed using some statistical divergence or metrics, such as the Prokhorov metric \citep{hafner1982simple}, KL divergence \citep{gul2017minimax}, $\alpha$-divergence \citep{gul2016robust}, Hellinger distance \citep{lecam1973convergence,birge1979theoreme}, Wasserstein distance \citep{gao2018robust,xie2021robust}, Maximum Mean Discrepancy \citep{sun2023kernel} and kernel Stein discrepancy \citep{liu2025robustness}. Other choices of neighbourhood includes density bands \citep{kassam1981robust,fauss2016old} and parameteric models \citep{key2025composite}. Robust testing differs from equivalence testing in that the null hypothesis in robust testing still assumes the data-generating distribution is close to the nominal distribution, rather than far away from it; see \Cref{fig:hypotheses} for an illustration. As a result, all these works still aim for testing model deviation rather than model evidence, thus differing from our framework.

Another related line of work is \emph{credal tests} \citep{chau2024credal}, which aim to detect the relationship between two arbitrary distributional simplexes. \citet{chau2024credal} proposed various MMD-based tests covering a range of scenarios. For example, their ``specification test'' targets null hypotheses where $P$ lies inside a probability simplex, and thus it can be viewed as an instance of robust tests. However, none of their tests target the equivalence hypothesis in \eqref{eq:equiv_hyps}. 

\section{One-Sample Equivalence Tests with KSD}
\label{sec:one_sample}
We first focus on one-sample testing, where we assume an i.i.d.\ sample $\X_n = \{X_i\}_{i=1}^n$ from $Q$ is observed, and sampling from $P$ is either infeasible or prohibitively expensive. Instead, we assume we can evaluate the score function $s_p(x)$ of $P$ at any given $x \in \R^d$. We will use KSD to construct equivalence tests for this setting and show their calibration and consistency.

We will assume the following conditions, which together ensure the Stein kernel $u_p$ has a finite second moment, which is sufficient for KSD to be $P$-separating and for the KSD estimator \eqref{eq:ksd_vstat} to have well-defined asymptotic distributions.
\begin{assumption}
\label{assump:kernel_universal}
    The reproducing kernel $k$ is \emph{$\cC_0^1$-universal} \citep[Definition 4.1]{carmeli2010vector}.
\end{assumption}

\begin{assumption}
\label{assump:score_condition}
    $ \E_{X \sim Q}[\| s_p(X) \|_2^4] < \infty$.
\end{assumption}

\begin{assumption}
\label{assump:ksd_kernel_condition}
    The kernel $k \in \cC^{(1,1)}_b$ satisfies $\E_{X \sim Q}[ | \nabla_1^\top \nabla_2 k(X, X) |^2 ] < \infty$.
\end{assumption}
\Cref{assump:score_condition} and \ref{assump:ksd_kernel_condition} are for simplicity rather than necessity and can be weakened. For example, they can be replaced by the sufficient condition $\Var_{X \sim Q}[|u_p(X, X')|^2] < \infty$, which will be made clear in the proof of \Cref{prop:ksd_asymptotics} below.

Under these assumptions, the asymptotic distributions of the KSD estimator \eqref{eq:ksd_vstat} is well-known by using classic results on V-statistics \citep[Section 5.5]{serfling2009approximation}. We summarize it in the following result, which is proven in \Cref{pf:prop:ksd_asymptotics}.

\begin{proposition}[KSD asymptotics]
\label{prop:ksd_asymptotics}
    Suppose \Cref{assump:kernel_universal}, \ref{assump:score_condition} and \ref{assump:ksd_kernel_condition} hold. 
    \begin{enumerate}
        \item If $Q \neq P$, then $\sqrt{n}(\KSD^2(Q_n, P) - \KSD^2(Q, P)) \darrow \cN(0, \sigma^2_\KSD)$ as $n \to \infty$, where $\sigma_\KSD^2 = 4 \Var_{X \sim Q}(\E_{X' \sim Q}[ u_p(X, X') ])$.
        \item If $Q = P$, then $\sigma_\KSD^2 = 0$, and $\sqrt{n} \cdot \KSD^2(Q_n, P) \to 0$ in probability.
        % \begin{align}
        %     n \cdot \KSD(Q_n, P)
        %     \;\darrow\;
        %     \sum_{j=1}^\infty \lambda_j (\xi_j^2 - 1)
        %     \;,
        %     \label{eq:ksd_degenerate_limit}
        % \end{align}
        % where $\{Z_j\}_{j=1}^\infty$ is a sequence of $\cN(0, 1)$ random variables, and $\lambda_j$ are eigenvalues of the kernel $u_p$, namely the solutions to the equation $\lambda_j \phi_j(\cdot) = \E_{X' \sim Q}[ u_p(\cdot, X') \phi_j(X')]$ for non-zero $\phi_j$.
    \end{enumerate}
\end{proposition}
A similar result is also shown in \citet[Theorem 4.1]{liu2016kernelized} but under stronger assumptions. For example, they assumed the data-generating distribution $Q$ also admits a Lebesgue density, a condition we do not assume in \Cref{prop:ksd_asymptotics}. Moreover, in fact, when $Q = P$, the V-statistic $\KSD^2(Q_n, P)$ becomes \emph{degenerate of order one} \citep[Section 2]{arcones1992bootstrap}, and it can be shown that $n \cdot \KSD^2(Q_n, P)$ (scaled by $n$ instead of $\sqrt{n}$) converges weakly to an infinite sum of weighted chi-squares \citep[Theorem 6.4.1 B]{serfling2009approximation}.

\subsection{A Normal Test}
\label{sec:one_sample:normal_test}
\Cref{prop:ksd_asymptotics} states that $\sqrt{n} \KSD^2(Q_n, P)$ is asymptotically normal when $Q \neq P$, and converges to $0$ when $Q = P$. This motivates an equivalence test based on normal approximations, where we reject $H_0: \KSD(Q, P) > \theta$ if $\KSD(Q_n, P)$ is smaller than some critical values computed from a normal quantile. We name it the \emph{E-KSD-Normal} test. Mathematically, this test rejects $H_0: \KSD(Q, P) > \theta$ if $S_n^{\KSD, \theta} < z_\alpha$, with
\begin{align}
    S_n^{\KSD, \theta}
    \;\coloneqq\;
    \frac{\sqrt{n}}{\hat{\sigma}_\KSD} (\KSD^2(Q_n, P) - \theta^2)
    \;,
    \label{eq:ksd_normal_test}
\end{align}
where $z_\alpha$ is the $\alpha$-th quantile of $\cN(0, 1)$, and $\hat{\sigma}_\KSD$ is the square root of the following leave-one-out Jackknife estimator for $\sigma^2_\KSD$
\begin{align}
    \hat{\sigma}_\KSD^2
    \;=\;
    \frac{4}{n-1} \sum_{i=1}^n \bigg( r_i - \frac{1}{n} \sum_{i'=1}^n r_{i'} \bigg)^2
    \;,\qquad
    r_i
    \;=\;
    \frac{1}{n-1} \sum_{i' \neq i} u_p(x_i,x_{i'})
    \;.
    \label{eq:ksd_jackknife}
\end{align}
Variance estimators similar to the above form have been used in the goodness-of-fit testing literature to approximate the power of KSD test \citep{liu2020learning,schrab2022ksd,liu2023using}. Using \Cref{prop:ksd_asymptotics}, the validity and consistency for this test can be shown. The proof is in \Cref{pf:thm:ksd_normal_test} and follows a similar argument in \citet[Section 2.3]{chen2023testing}.

\begin{theorem}[E-KSD-Normal test]
\label{thm:ksd_normal_test}
    Suppose \Cref{assump:kernel_universal} holds. Then, for all $Q$ satisfying \Cref{assump:score_condition} and \ref{assump:ksd_kernel_condition},
    \begin{align*}
        \lim_{n \to \infty} \Pr(S_n^{\KSD,\theta} > \gamma_{1-\alpha})
        \;=\;
        \begin{cases}
            0 \;, &\KSD(Q, P) > \theta \;, \\
            \alpha \;, &\KSD(Q, P) = \theta \;, \\
            1 \;, & \KSD(Q, P) < \theta \;.
        \end{cases}
    \end{align*}
\end{theorem}
\Cref{thm:ksd_normal_test} immediately implies that the Type-I error, namely $\lim_{n \to \infty} \Pr(S_n^{\KSD,\theta} > \gamma_{1-\alpha})$, is bounded by $\alpha$ under the null hypothesis $H_0: \KSD(Q, P) \geq \theta$, while converges to 1 under the alternative hypothesis $H_1: \KSD(Q, P) < \theta$. These combined mean that the E-KSD-Normal test is both well-calibrated and consistent.

\begin{remark}
    Our proof follows a standard technique in the equivalence testing literature, applicable when the test statistic is asymptotically normal; see, e.g., \citet{freitag2007nonparametric,dette2018equivalence,chen2023testing,baillo2024almost}. Most closely related to our result is \citet[Section 2.3]{chen2023testing}. They constructed their test using MMD and assumed they have access to samples of \emph{equal} size from \emph{both} $Q$ and $P$. The equal-sample-size assumption allows the test statistic to be expressed as a one-sample V-statistic, similar to our formulation with KSD. Our proof can therefore be viewed as a direct counterpart to theirs, with MMD replaced by KSD. Nevertheless, we highlight that the bootstrapped test introduced in the next section is a fundamentally novel and distinct approach.
\end{remark}

\begin{remark}
    The Type-I control in \Cref{thm:ksd_normal_test} is \emph{pointwise} over the null set, namely the limit over $n$ is taken for a fixed null distribution. In general, for hypothesis tests targeting a composite null set, it is more desirable to have a \emph{uniform} Type-I error control, namely $\lim\sup_{n \to \infty} \sup_{\{Q: \KSD(Q, P) \geq \theta\}}\Pr(S_n^{\KSD,\theta} > \gamma_{1-\alpha}) \leq \alpha$; see \citet[Chapter 11]{lehmann2008testing} for a detailed discussion. However, most work in the equivalence testing literature concerns pointwise control \citep{dette2018equivalence,chen2023testing}, which is why we also prove pointwise control for our proposed tests in this work.
\end{remark}

% --- with p-values ---
% \begin{algorithm}[t]
%     \caption{Normal equivalence testing with KSD.}
%     \label{alg:ksd_normal_test}
%     \begin{algorithmic}[1]
%         \State {\bfseries Inputs:} Observations $\X_n = \{x_i\}_{i=1}^n$; test level $\alpha$; bootstrap sample size $B$.
%         \State Compute the test statistic $S_{n}^{\KSD, \theta}$ using \eqref{eq:ksd_normal_test}.
%         \State Compute the $p$-value $\hat{p}=\Phi(S_{n}^{\KSD, \theta})$, where $\Phi(\cdot)$ is the CDF of $\cN(0, 1)$, and reject $H_0$ if $\hat{p} < \alpha$.
%     \end{algorithmic}
% \end{algorithm}

% \begin{algorithm}[t]
%     \caption{Bootstrapped equivalence testing with KSD.}
%     \label{alg:ksd_bootstrap_test}
%     \begin{algorithmic}[1]
%         \State {\bfseries Inputs:} Observations $\X_n = \{x_i\}_{i=1}^n$; test level $\alpha$; bootstrap sample size $B$.
%         \State Compute the test statistic $T_{n}^{\KSD, \theta} = \theta - \KSD(Q_n, P)$.
%         \For{$b = 1, \ldots, B$}
%         \State Draw independent copies $W_n^b \sim \mathrm{Multinomial(n; 1/n, \ldots, 1/n)}$.
%         \State Compute bootstrap samples $T_{n, b}^{\KSD, \theta} = \theta - D_{W_n^b}(\X_n; u_p)$ using \eqref{eq:ksd_bootstrap_sample}.
%         \EndFor
%         \State Compute the $p$-value $\hat{p}=\frac{1}{B} \sum_{b=1}^B \indicator\{ T_{n,b}^{\KSD, \theta} \geq T_n^{\KSD, \theta} \}$ and reject $H_0$ if $\hat{p} < \alpha$. 
%     \end{algorithmic}
% \end{algorithm}

% --- with quantiles ---
\begin{algorithm}[t]
    \caption{Normal equivalence testing with KSD.}
    \label{alg:ksd_normal_test}
    \begin{algorithmic}[1]
        \State {\bfseries Inputs:} Observations $\X_n = \{x_i\}_{i=1}^n$; test level $\alpha$; bootstrap sample size $B$.
        \State Compute the test statistic $S_{n}^{\KSD, \theta}$ using \eqref{eq:ksd_normal_test}.
        \State Reject $H_0$ if $S_{n}^{\KSD, \theta} < z_\alpha$, where $z_\alpha$ is the $\alpha$-th quantile of $\cN(0, 1)$.
    \end{algorithmic}
\end{algorithm}

\begin{algorithm}[t]
    \caption{Bootstrapped equivalence testing with KSD.}
    \label{alg:ksd_bootstrap_test}
    \begin{algorithmic}[1]
        \State {\bfseries Inputs:} Observations $\X_n = \{x_i\}_{i=1}^n$; test level $\alpha$; bootstrap sample size $B$.
        \State Compute the test statistic $T_{n}^{\KSD, \theta} = \theta - \KSD(Q_n, P)$.
        \For{$b = 1, \ldots, B$}
        \State Draw independent copies $W_n^b \sim \mathrm{Multinomial(n; 1/n, \ldots, 1/n)}$.
        \State Compute bootstrap samples $T_{n, b}^{\KSD, \theta} = \theta - D_{W_n^b}(\X_n; u_p)$ using \eqref{eq:ksd_bootstrap_sample}.
        \EndFor
        \State Compute the $(1-\alpha)$-th empirical quantile, $\gamma_{1-\alpha}^B$, using \eqref{eq:ksd_bootstrap_quantile}, and reject $H_0$ if $T_n^{\KSD, \theta} > \gamma_{1-\alpha}^B$.
    \end{algorithmic}
\end{algorithm}

\subsection{A Bootstrapped Test}
\label{sec:one_sample:bootstrap_test}
The E-KSD-Normal test proposed in \Cref{sec:one_sample:normal_test} is guaranteed to control Type-I error in the infinite sample size limit, as shown in \Cref{thm:ksd_normal_test}. 
However, as will be shown empirically in \Cref{sec:exp:gauss_ms}, the E-KSD-Normal test can have poor Type-I error control when the margin $\theta$ is small, even if the sample size is moderately large. 

This is because the KSD estimator has different limiting distributions in the two regimes $Q = P$ and $Q \neq P$. Too see this, assume $\KSD(Q, P) = \theta$, so that $Q$ is on the boundary of the null set. \Cref{prop:ksd_asymptotics} shows that the (scaled) KSD estimator \eqref{eq:ksd_vstat} is asymptotically normal when $\theta > 0$, while it converges to 0 when $\theta = 0$. In fact, it is well-known that in the latter case, \eqref{eq:ksd_vstat} converges to an infinite sum of weighted chi-squares at rate $n^{-1}$; see, e.g., \citet[Theorem 4.1]{liu2016kernelized} and \citet{huang2023highdimensional}. In particular, when $\theta$ is small (but non-zero), the normal approximation of \eqref{eq:ksd_vstat} would deteriorate. Moreover, the variance estimator \eqref{eq:ksd_jackknife}, although a common choice in the kernel-based testing literature \citep{bounliphone2016kernel,jitkrittum2018informative}, can underestimate the actual variance \citep[Appendix C4]{kanagawa2023kernel}. These issues combined can result in an inflated Type-I error. Notably, such Type-I error inflation is not unique to our kernel-based test. Other equivalence tests constructed with asymptotic normality have also found to suffer from this issue; see, e.g., the discussions in  \citet[Section 3]{dette2018equivalence} and \citet[Section 4]{chen2023testing}.

Motivated by this observation, we propose an alternative test based on bootstrapping. This bootstrapping-based test is constructed using a completely different principle from the E-KSD-Normal test. In particular, it does not rely on normal approximations to the KSD estimator. Instead, it uses a triangle inequality for the KSD \citep{shi2024finiteparticle,liu2025robustness} to derive a more conservative (yet still consistent) upper bound, thus achieving a better Type-I error control, especially with small $\theta$. 

We first present the test procedure and provide intuition later. This test, which we call the \emph{E-KSD-Boot} test, rejects $H_0: \KSD(Q, P) \geq \theta$ if $T_n^{\KSD, \theta} > \gamma_{1-\alpha}^B$, where
\begin{align}
    T_n^{\KSD, \theta}
    \;\coloneqq\;
     \theta - \KSD(Q_n, P)
     \;,
    \label{eq:ksd_boot_test_stat}
\end{align}
and rejects $H_0: \KSD(Q, P) \geq \theta$ if $T_n^{\KSD, \theta} > \gamma^B_{1-\alpha}$, where $\gamma^B_{1-\alpha}$ is defined in \eqref{eq:ksd_bootstrap_quantile}. The full procedure is summarized in \Cref{alg:ksd_bootstrap_test}.

We now provide some intuition for the E-KSD-Boot test. The test statistic $T_n^{\KSD, \theta}$ quantifies how close the KSD value is from the pre-specified margin $\theta$. Since the null hypothesis $H_0$ assumes that $\KSD(Q, P) > \theta$, we should reject $H_0$ for small values of $\KSD(Q_n, P)$, or, equivalently, for large values of $T_n^{\KSD, \theta}$. 

A natural question is then how to choose a critical value? This critical value must control the Type-I error rate for \emph{every} possible $Q$ under the equivalence null hypothesis. In the standard KSD test reviewed in \Cref{sec:ksd}, where the test statistic is $\KSD(Q_n, P)$, the critical value is its $(1-\alpha)$-th quantile, $\gamma_{1-\alpha}$, under the point null hypothesis \eqref{eq:standard_hyps}, and weighted bootstrapping is used to approximate $\gamma_{1-\alpha}$. However, it is unclear whether a similar approach can be used in this case, as the null set under the equivalence null hypothesis $H_0: \KSD(Q, P) \geq \theta$ is no longer a singleton, but an equivalence set containing sufficiently distinct distributions.

Our major contributions of this section are to show that $\gamma_{1-\alpha}$ is a valid critical value even for such equivalence null hypothesis, and that the same weighted bootstrap approached can be used as an approximation. We first present the following lemma, which shows that the bootstrap samples \eqref{eq:ksd_bootstrap_sample} approximate the distribution of an \emph{MMD} statistic $\MMD(Q_n, Q; k)$. The proof is in \Cref{pf:lem:one_sample_bootstrap_validity}.
\begin{lemma}
\label{lem:one_sample_bootstrap_validity}
    Let $\X_\infty = \{ X_i \}_{i=1}^\infty$ be a random sample where $X_i \sim Q \in \cP$ are independent, and for any $n$ let $Q_n$ be the empirical measure based on $\X_n = \{X_i\}_{i=1}^n$. Let $k$ be a positive definite kernel such that $\E_{X, X' \sim Q}[| k(X, X')|^2] < \infty$. Then, for all $Q$-almost-sure sequences $\X_\infty$, the following holds
    \begin{align*}
        \sup_{t \in \R} \big|
            \Pr( \sqrt{n} D_{W_n}(\X_n; k) \leq t \;|\; \X_\infty )
            - \Pr( \sqrt{n} \MMD(Q_n, Q; k) \leq t )
        \big| \;\to\; 0
        \;.
    \end{align*}
\end{lemma}

\begin{remark}
    \label{rem:pksd}
    \citet[Lemma 20]{liu2025robustness} showed this result for the special case where $k$ is a Stein reproducing kernel. The authors referred to $\MMD(Q_n, Q; u_p)$ as the \emph{$P$-KSD}. We follow a similar proof strategy by noting that their proof directly extends to a general reproducing kernel $k$. The proof is included in \Cref{pf:lem:one_sample_bootstrap_validity} for completeness. This generalization will also become helpful later when we study a similar equivalence test based on MMD.
\end{remark}

\begin{remark}
    \Cref{lem:one_sample_bootstrap_validity} is a consequence of the bootstrapping results for each term in the \emph{Hoeffding decomposition} of V-statistics \citep{hoeffding1948class,arcones1992bootstrap}. As previously mentioned, this is \emph{not} a direct application of standard bootstrapping results for the degenerate V-statistics $\MMD^2(Q_n, Q)$, since the bootstrapping statistic so obtained would involve intractable expectations and thus cannot be computed. In the literature, the result in \Cref{lem:one_sample_bootstrap_validity} was presented in \citet[Theorem 2.1]{dehling1994random} in a slightly modified form, but their proof only applies to univariate distributions. The general case for multivariate distributions is implicitly shown in the proof of \citet[Theorem 2.4]{arcones1992bootstrap}. We provide a clearer exposition and close the gap.
\end{remark}

It turns out that \Cref{lem:one_sample_bootstrap_validity} is sufficient to establish the validity and consistency of the E-KSD-Boot test, which is summarized in the following result. Intuitively, this is because, under the equivalence null hypothesis, the test statistic \eqref{eq:ksd_boot_test_stat} can be bounded using statistics of the form $\MMD(Q_n, Q; k)$, an observation that will be made clear in the proof in \Cref{pf:thm:ksd_bootstrap_test}.
\begin{theorem}[E-KSD-Boot test]
    \label{thm:ksd_bootstrap_test}
    Let $\alpha \in (0, 1)$ and, for any $n$, let $W \sim \mathrm{Multinomial}(n, 1/n, \ldots, 1/n)$. Denote by $\gamma_{1-\alpha}^\infty$ the $(1-\alpha)$-quantile of the conditional distribution of $D_W^2$, defined in \eqref{eq:ksd_bootstrap_sample}, given $\X_n$. Suppose \Cref{assump:kernel_universal}, \ref{assump:score_condition} and \ref{assump:ksd_kernel_condition} hold. Then, for all $Q \in \cP(\R^d)$ with $\E_{X \sim Q}[\| s_p(X) \|_2] < \infty$, there exists $a \in (0, \alpha]$ such that
    \begin{align*}
        \lim_{n \to \infty} \Pr(T_n^{\KSD, \theta} > \gamma_{1-\alpha}^\infty)
        \;=\;
        \begin{cases}
            0 \;, &\KSD(Q, P) > \theta \;, \\
            a \;, &\KSD(Q, P) = \theta \;, \\
            1 \;, & \KSD(Q, P) < \theta \;.
        \end{cases}
    \end{align*}
\end{theorem}
In other words, \Cref{thm:ksd_bootstrap_test} implies that the probability of rejection is asymptotically no larger than $a \leq \alpha$ under $H_0: \KSD(Q, P) \geq \theta$, while it converges to $1$ under $H_1: \KSD(Q, P) < \theta$. 

\begin{remark}
    Key to the proof of \Cref{thm:ksd_bootstrap_test} is the following upper bound of KSD using MMD and $P$-KSD (see \Cref{rem:pksd})
    \begin{align*}
        \KSD(Q, P)
        \;=\;
        \MMD(Q, P; u_p)
        \;&\leq\;
        \MMD(Q, Q_n; u_p) + \MMD(Q_n, P; u_p)
        \\
        \;&=\;
        \MMD(Q, Q_n; u_p) + \KSD(Q_n, P)
        \tagaligneq
        \label{eq:KSD_traingle_ineq}
        \;.
    \end{align*}
    The first equality uses the fact that KSD can be rewritten as an MMD with kernel $u_p$, namely $KSD(Q, P) = \MMD(Q, P; u_p)$, and the inequality leverages the triangle inequality for MMD. A more rigorous proof for \eqref{eq:KSD_traingle_ineq} is given in \Cref{pf:thm:ksd_bootstrap_test}. This inequality resembles a triangle inequality for KSD. Similar inequalities have been studied in \citet{shi2024finiteparticle} to study the convergence rate of kernel-based interacting particle systems, and in \citet[Appendix A.4]{liu2025robustness} to construct robust GOF tests with KSD. Our proof is motivated by \citet{liu2025robustness} and is adjusted to equivalence hypotheses.
\end{remark}

The E-KSD-Boot test bears similarity with the \emph{robust-KSD test} of \citet{liu2025robustness}. In \citet{liu2025robustness}, they target robust hypotheses of the form $H_0^R: \KSD(Q, P) \leq \theta$ versus $H_1^R: \KSD(Q, P) > \theta$, and their test rejects $H_0^R$ if $\KSD(Q_n, P) - \theta > \gamma_{1-\alpha}^B$, where $\gamma_{1-\alpha}^B$ is the bootstrapped quantile \eqref{eq:ksd_bootstrap_quantile}. Notably, we consider the ``reversed'' hypotheses \eqref{eq:equiv_hyps}, and our test statistic \eqref{eq:ksd_boot_test_stat} differs by a negative sign to reflect the fact that the equivalence null hypothesis in \eqref{eq:equiv_hyps} now should be rejected for \emph{small}, rather than large, values of $\KSD(Q_n, P)$. The theoretical proofs for our E-KSD-Boot test essentially shows that the same threshold used in the test of \citet{liu2025robustness} remains valid after such reversal.

\section{Two-Sample Equivalence Tests with MMD}
\label{sec:two_sample}

\begin{algorithm}[t]
    \caption{Normal equivalence testing with MMD.}
    \label{alg:mmd_normal_test}
    \begin{algorithmic}[1]
        \State {\bfseries Inputs:} Observations $\X_n = \{x_i\}_{i=1}^n$ and $\Y_m = \{y_j\}_{j=1}^m$; test level $\alpha$; bootstrap sample size $B$.
        \State Compute the test statistic $S_{n,m}^{\MMD, \theta}$ using \eqref{eq:mmd_normal_test}.
        \State Reject $H_0$ if $S_n^{\MMD, \theta} < z_\alpha$, where $z_\alpha$ is the $\alpha$-th quantile of $\cN(0, 1)$.
    \end{algorithmic}
\end{algorithm}

\begin{algorithm}[t]
    \caption{Bootstrapped equivalence testing with MMD.}
    \label{alg:mmd_bootstrap_test}
    \begin{algorithmic}[1]
        \State {\bfseries Inputs:} Observations $\X_n = \{x_i\}_{i=1}^n$ and $\Y_m = \{y_j\}_{j=1}^m$; test level $\alpha$; bootstrap sample size $B$.
        \State Compute the test statistic $T_{n,m}^{\MMD, \theta} = \theta - \MMD(Q_n, P_m)$.
        \For{$b = 1, \ldots, B$}
        \State Draw independent copies $W_n^b \sim \mathrm{Multinomial(n; 1/n, \ldots, 1/n)}$ and $\widetilde{W}_m^b \sim \mathrm{Multinomial(m; 1/m, \ldots, 1/m)}$.
        \State Compute bootstrap samples $T_{n,m, b}^{\MMD, \theta} = \theta - S_{n,m}^b = \theta - D_{W_n^b}(\X_n) - D_{\widetilde{W}_n^b}(\Y_m)$ using \eqref{eq: bootstrap sample}.
        \EndFor
        \State Compute the $(1-\alpha)$-th empirical quantile, $\eta_{1-\alpha}^B$, using \eqref{eq: bootstrap quantile}, and reject $H_0$ if $T_n^{\MMD, \theta} > \eta_{1-\alpha}^B$.
    \end{algorithmic}
\end{algorithm}

In \Cref{sec:one_sample}, we assumed that the target distribution $P$ is accessible through its score function. In some cases, however, even evaluation of the score function can be infeasible; instead, it is straightforward to draw finite samples from $P$ as an approximation. Many generative models, such as \emph{generative adversarial networks} \citep[GANs,][]{goodfellow2014generative} and many simulator-based models \citep{cranmer2020frontier}, are defined via data-generating processes and do not necessarily have a tractable score function. In such cases, KSD is no longer an appropriate statistical divergence, and the KSD-based equivalence tests are no longer applicable.

% Next, we turn our focus to two-sample testing, where we assume we observe independent, random samples $\X_n$ from $Q$ and $\Y_m$ from $P$. Such assumptions are natural in applications such as simulation-based inference \citep{cranmer2020frontier}, where models are often defined via some data-generating process and thus even their unnormalized densities are intractable. 

To address these challenges, we propose equivalence tests based on Maximum Mean Discrepancy. As discussed in \Cref{sec:mmd}, MMD can be estimated efficiently given finite samples from both $Q$ and $P$, and it does not require evaluations of the score function of $P$. The proposed tests still target hypotheses \eqref{eq:equiv_hyps}, except that the discrepancy $D$ is now chosen to be MMD. Importantly, since $D$ has changed, the null set is now different from that considered in the previous section, and in particular it might include vastly different distributions. Similarly to \Cref{sec:one_sample}, two equivalence tests are studied, one using a normal approximation, and the other using a bootstrapping approach.

We make the following assumptions in this section. \Cref{assump:mmd_moment} is sufficient for the MMD estimator \eqref{eq:mmd_v_stat} to have well-defined asymptotic distributions and is common in the kernel testing literature \citep{gretton2012kernel,shekhar2022permutation}. In particular, it holds when the kernel is bounded. \Cref{assump:mmd_sample_sizes} is a mild condition requiring the two sample sizes to grow at a specific rate.
\begin{assumption}
\label{assump:mmd_moment}
    $\E_{X, X' \sim Q}[| k(X, X') |^2] < \infty$ and $\E_{Y, Y' \sim P}[| k(Y, Y') |^2] < \infty$.
\end{assumption}

\begin{assumption}
\label{assump:mmd_sample_sizes}
    Let $N = n + m$. There exists $\nu \in (0, 1)$ such that $n / N \to \nu$ as $n \to \infty$.
\end{assumption}
Under these assumptions, the asymptotics of the finite-sample estimator \eqref{eq:mmd_v_stat_two_sample}, which is a two-sample V-statistics of degree $(2,2)$, are well-known. We summarize them in the next result for completeness. Its proof is in \Cref{pf:prop:mmd_asymptotics}.

\begin{proposition}[MMD asymptotics]
\label{prop:mmd_asymptotics}
    Suppose \Cref{assump:mmd_moment} and \Cref{assump:mmd_sample_sizes} hold.
    \begin{enumerate}
        \item If $Q \neq P$, then $N^{1/2} (\MMD^2(Q_n, P_m) - \MMD^2(Q, P)) \darrow \cN(0, \sigma_\MMD^2)$ as $N \to \infty$, where $\sigma_\MMD^2 = 4N(n^{-1} \sigma_1^2 + m^{-1}\sigma_2^2)$ is positive, and
        \begin{align*}
            \sigma_1^2
            \;=\;
            \Var_{X \sim Q}( \E_{X' \sim Q, Y, Y' \sim P}[ h(X, X', Y, Y') ] )
            \;,\quad
            \sigma_2^2
            \;=\;
            \Var_{Y \sim P}( \E_{X, X' \sim Q, Y' \sim P}[ h(X, X', Y, Y') ] )
            \;.
        \end{align*}
        \item If $Q = P$, then $\sigma_\MMD^2 = 0$, and $ \sqrt{N} \cdot \MMD^2(Q_n, P_m) \to 0$ in probability.
    \end{enumerate}
\end{proposition}

\subsection{A Normal Test}
\label{sec:two_sample:normal_test}
\Cref{prop:mmd_asymptotics}, combined with the fact that MMD is a metric, implies that the statistic $\sqrt{N} (\MMD^2(Q_n, P_m) - \MMD^2(Q, P))$ is asymptotically normal under $H_0: \MMD(Q, P) \geq \theta$. We can hence follow the same principle in \Cref{sec:one_sample:normal_test} to construct an equivalence test based on normal approximations. A special case of \Cref{prop:mmd_asymptotics} with $n = m$ motivated the equivalence test of \citet{chen2023testing}. Our test can be viewed as a generalization of theirs to $n \neq m$. Such extension is natural, since $m$, which is the sample size for the empirical sample drawn from the target distribution $P$, often needs to be considerably larger than $n$ to ensure minimal error due to finite-sample approximation of $P$. However, it is sometimes infeasible to augment the sample size $n$, for example, when generating more samples is computationally prohibitive, or even impossible due to experimental limitations.

The proposed test, which we call the \emph{E-MMD-Normal} test, rejects $H_0$ if $S_{n,m}^{\MMD, \theta} < z_\alpha$, where
\begin{align}
    S_{n,m}^{\MMD, \theta} \coloneqq \frac{\sqrt{N}}{\hat{\sigma}_\MMD}(\MMD^2(Q_n, P_m) - \theta^2)
    \;,
    \label{eq:mmd_normal_test}
\end{align}
and $z_\alpha$ is the $\alpha$-th quantile of $\cN(0, 1)$, and $\hat{\sigma}_\MMD$ is the square root of an estimator for $\sigma_\MMD^2$ defined as $\hat{\sigma}_\MMD^2 = N(4n^{-1} \hat{\sigma}_{\MMD,1}^2 + 4m^{-1} \hat{\sigma}_{\MMD,2}^2)$, with
\begin{align}
    \hat{\sigma}_{\MMD,1}^2
    \;=\;
    \frac{1}{n-1}\sum_{i=1}^n\bigg(q_i - \frac{1}{n} \sum_{i'=1}^n q_{i'} \bigg)^2
    \;,\qquad
    \hat{\sigma}_{\MMD,2}^2
    \;=\;
    \frac{1}{m-1}\sum_{j=1}^m\bigg(p_i - \frac{1}{m} \sum_{j'=1}^m p_{j'} \bigg)^2
    \;,
    \label{eq:mmd_normal_vars}
\end{align}
and
\begin{align*}
    q_i \;&=\; \frac{1}{(n-1)m(m-1)} \sum_{\substack{i' = 1\\ i' \neq i}}^n \sum_{j=1}^m \sum_{\substack{j'=1\\ j' \neq j}}^m h(x_i, x_{i'}, y_j, y_{j'})
    % \;,\quad
    \\
    p_j \;&=\; \frac{1}{n(n-1)(m-1)} \sum_{i=1}^n \sum_{\substack{i'=1\\ i' \neq i}}^{n} \sum_{\substack{j' = 1\\ j
    ' \neq j}}^m h(x_i, x_{i'}, y_j, y_{j'})
    \;.
\end{align*}
In particular, $\hat{\sigma}_{\MMD,1}^2$ and $\hat{\sigma}_{\MMD,2}^2$ are leave-one-out Jackknife estimators for $\sigma_1^2, \sigma_2^2$, respectively. The studentized statistic $S_{n,m}^{\MMD, \theta}$ is known to converge weakly to a standard Gaussian distribution \citep{chang2016cramertype}. Moreover, these Jackknife estimators can be computed in quadratic time; see \Cref{app:mmd_normal_vars} for details.

The next result shows that the E-MMD-Normal test, summarized in \Cref{alg:mmd_normal_test}, is well-calibrated and consistent. Its proof is in \Cref{pf:thm:mmd_normal_test}.
\begin{theorem}[E-MMD-Normal test]
\label{thm:mmd_normal_test}
    Suppose \Cref{assump:mmd_moment} and \Cref{assump:mmd_sample_sizes} hold. Let $\theta > 0$ and let $z_\alpha$ be the $\alpha$-th quantile of a standard normal distribution. Then, for all $Q, P \in \cP(\R^d)$,
    % \begin{enumerate}
    %     \item If $\MMD(Q, P) > \theta$, then $\lim_{n \to \infty} \Pr(S_{n,m}^{\MMD, \theta} > \gamma_{1-\alpha}) = 0$.
    %     \item If $\MMD(Q, P) = \theta$, then $\lim_{n \to \infty} \Pr(S_{n,m}^{\MMD, \theta} > \gamma_{1-\alpha}) = \alpha$.
    %     \item If $\MMD(Q, P) < \theta$, then $\lim_{n \to \infty} \Pr(S_{n,m}^{\MMD, \theta} > \gamma_{1-\alpha}) = 1$.
    % \end{enumerate}
    \begin{align*}
        \lim_{N \to \infty} \Pr(S_{n,m}^{\MMD, \theta} > \gamma_{1-\alpha})
        \;=\;
        \begin{cases}
            0 \;, &\MMD(Q, P) > \theta \;, \\
            \alpha \;, &\MMD(Q, P) = \theta \;, \\
            1 \;, & \MMD(Q, P) < \theta \;.
        \end{cases}
    \end{align*}
\end{theorem}

% Since $\MMD(Q, P) \geq \theta > 0$ under the equivalence null hypothesis $H_0$ and $0 \leq \MMD(Q, P) < \theta$ under the alternative hypothesis $H_1$, \Cref{thm:mmd_normal_test} shows that the MMD normal equivalence test is well-calibrated and consistent. A similar result is proven in \citet[Section 2.3]{chen2023testing} for the case $m = n$. Our proof generalizes their result to unequal sample sizes by using asymptotics for two-sample, instead of one-sample U-statistics.

\begin{remark}[Comparison to \citet{chen2023testing}]
    Our proposed E-MMD-Normal test resembles the homogeneity-equivalence test of \citet{chen2023testing} in that both are constructed using the asymptotic normality of MMD, and that both aim to test equivalence hypotheses of the form \eqref{eq:equiv_hyps}. However, \citet{chen2023testing} assume equal sample size $m=n$, so their test statistic is a \emph{one-sample} U-statistic. In contrast, we allow the more general case where $m$ and $n$ may differ, and thus our test statistic are \emph{two-sample} V-statistics. On the other hand, the difference between using U-statistics and V-statistics is minor, and the same proof technique used in  \Cref{thm:mmd_normal_test} can be applied to establish a similar result for two-sample U-statistics.
\end{remark}

\subsection{A Bootstrapped Test}
\label{sec:two_sample:bootstrap_test}
Since the test proposed in the previous section is based on an asymptotical normal approximation of its test statistic, it suffers from the same issue as the normality-based KSD equivalence test proposed in \Cref{sec:one_sample:normal_test}, namely it can fail to control Type-I error when the equivalence margin $\theta$ is small. Next, we describe a bootstrapped test, which is a counterpart of the E-KSD-Boot test proposed in \Cref{sec:one_sample:bootstrap_test} for MMD. The following bootstrapped test, called \emph{E-MMD-Boot}, rejects $H_0$ if $T_{n,m}^{\MMD, \theta} > \eta$, where
\begin{align}
    T_{n,m}^{\MMD, \theta}
    \;\coloneqq\;
     \theta - \MMD(Q_n, P_m)
     \;,
    \label{eq: MMD test evidence}
\end{align}
and $\eta > 0$ is a critical value chosen to control the Type-I error. 

The intuition behind the construction of the E-MMD-Boot is similar to the E-KSD-Boot test---the equivalence null hypothesis $H_0: \MMD(Q, P) \geq \theta$ is rejected for large values of $T_{n,m}^{\MMD, \theta}$. One difference is that the test statistic $T_{n,m}^{\MMD, \theta}$ now involves an MMD statistic, rather than a KSD statistic. Consequently, the bootstrapping approach described in \Cref{sec:one_sample:bootstrap_test} needs to be adjusted to compute the critical value. 
We will first discuss how to choose such $\eta$, and then prove the validity and consistency of this test at the end of this section.

Our key observation is that it suffices to choose the critical value as the $(1-\alpha)$-th quantile $\eta_{1-\alpha}$ of $S_{n,m} = \MMD(Q_n, Q) + \MMD(P_m, P)$. Since $\eta_{1-\alpha}$ is also tractable, we introduce a bootstrapping approach to estimate it. First, \Cref{lem:one_sample_bootstrap_validity} implies that each of the terms $\MMD(Q_n, Q)$ and $\MMD(P_m, P)$ can be individually bootstrapped by the square-root of the bootstrap sample \eqref{eq:ksd_bootstrap_sample} with the Stein kernel $u_p$ replaced by $k$. Specifically, the distribution of $\MMD(Q_n, Q)$ can be approximated using bootstrap samples of the form
\begin{align}
    D^2_{W_n}(\X_n)
    \;\coloneqq\;
    \frac{1}{n^2} \sum_{1 \leq i, j \leq n} (W_{ni} - 1) (W_{nj} - 1) k(x_i, x_j)
    \;,
    \label{eq: bootstrap sample}
\end{align}
where $W_n = (W_{n1}, \ldots, W_{nn}) \sim \textrm{Multinomial}(n; 1/n, \ldots, 1/n)$. The distribution of $\MMD(P_m, P)$ can be estimated similarly. Therefore, given two sets of i.i.d.\ copies $\{ W_n^b \}_{b=1}^B$ and $\{ \widetilde{W}_m^b \}_{b=1}^B$ of size $B$ drawn from $\textrm{Multinomial}(n; 1/n, \ldots, 1/n)$, we can compute the following bootstrap samples for $S_n$
\begin{align}
    S^b_{n,m}
    \;\coloneqq\; D_{W_n^b}(\X_n) + D_{\widetilde{W}_m^b}(\Y_m)
    \;,
    \label{eq:mmd_bootstrap_sample}
\end{align}
for $b = 1, \ldots, B$. We then approximate the quantile $\eta_{1-\alpha}$ of the distribution of $S_{n,m}$ by the quantile $\eta^B_{1-\alpha}$ of the bootstrap samples $\{ S_{n,m}^b \}_{b=1}^B$, i.e.,
\begin{align}
    \eta^B_{1-\alpha}
    \;=\;
    \inf\bigg\{u \in \R: \; \frac{1}{B} \sum_{b = 1}^{B} \indicator\{ S_{n,m}^b \leq u \} \geq 1 - \alpha \bigg\}
    \;.
    \label{eq: bootstrap quantile}
\end{align}
The validity of this bootstrap approach is justified by the following result, which is a direct consequence of \Cref{lem:one_sample_bootstrap_validity}. Its proof is included in \Cref{pf:cor: validity of bootstrap}. The full testing procedure is summarized in \Cref{alg:mmd_bootstrap_test}.
\begin{lemma}
\label{cor: validity of bootstrap}
    Let $\Z_\infty = \{ (X_i, Y_i) \}_{i=1}^\infty$ be a random sample where $X_i \sim Q$ and $Y_i \sim P$ are independent. For any $n, m$, let $Q_n$ and $P_m$ be the empirical measure based on $\X_n = \{X_i\}_{i=1}^n$ and $\Y_m = \{Y_j\}_{j=1}^m$, respectively. Assume $0 \leq k(x, y) \leq K$ for some constant $K > 0$, and suppose \Cref{assump:mmd_sample_sizes} holds. Define $S_{n,m}^\ast \coloneqq D_{W_n}(\X_n) + D_{\widetilde{W}_m}(\Y_m)$, where $D_{W_n}(\X_n)$ and $D_{\widetilde{W}_m}(\Y_m)$ are defined in \eqref{eq: bootstrap sample} and $W_n$ and $\widetilde{W}_m$ are independent. Then, for all $(Q\times P)$-almost-sure sequences $\Z_\infty$, it holds
    \begin{align*}
        \sup_{t \in \R} \big|
            \Pr( \sqrt{N} S_{n,m}^\ast \leq t \;|\; \Z_\infty )
            - \Pr( \sqrt{N} S_{n,m} \leq t )
        \big| \;\to\; 0
        \;.
    \end{align*}
\end{lemma}

Having introduced the bootstrap approach required to estimate the critical value, we now show our claim that, with this choice, the E-MMD-Boot test is both well-calibrated and consistent. The proofs is in \Cref{app:proofs:mmd}.
\begin{theorem}[E-MMD-Boot test]
    \label{thm:mmd_bootstrap}
    Let $\alpha \in (0, 1)$. Denote by $\eta_{1-\alpha}^\infty$ the $(1-\alpha)$-quantile of the conditional distribution of $S_{n,m}^\ast \coloneqq D_{W_n}(\X_n) + D_{\widetilde{W}_m}(\Y_m)$ given $\X_n$ and $\Y_m$, where $D_{W_n}(\X_n)$ and $D_{\widetilde{W}_m}(\Y_m)$ are defined in \eqref{eq: bootstrap sample} and $W_n$ and $\widetilde{W}_m$ are independent. Suppose \Cref{assump:mmd_moment} and \Cref{assump:mmd_sample_sizes} hold. Then, for all $Q, P \in \cP(\R^d)$, there exists $a \leq \alpha$ such that
    \begin{align*}
        \lim_{n \to \infty} \Pr(T_{n,m}^{\MMD, \theta} > \eta_{1-\alpha}^\infty)
        \;=\;
        \begin{cases}
            0 \;, &\MMD(Q, P) > \theta \;, \\
            a \;, &\MMD(Q, P) = \theta \;, \\
            1 \;, & \MMD(Q, P) < \theta \;.
        \end{cases}
    \end{align*}
\end{theorem}

\section{A Minimal-Effect Approach for Selecting the Equivalence Margin}
\label{sec:choice_of_equiv_margin}

The equivalence margin $\theta$ is a crucial design choice. Intuitively, it needs to be chosen so that the KSD or MMD balls includes the distributions that the user is willing to tolerate. In other words, it can be viewed as a form of model uncertainty. In general, designing an appropriate $\theta$ is a challenging task that has received ongoing attention from not only the equivalence testing literature \citep{simonsohn2015small,lakens2017equivalence,dette2018equivalence}, but also the closely related field of robust testing \citep{fauss2021minimax,liu2025robustness}. 
% In this section, we discuss some possible approaches to choose $\theta$ in practice.
In this section, we propose a data-driven approach, where $\theta$ is chosen to be the smallest effect size that the test can detect with a pre-specified power.

Given $Q$ such that $\KSD(Q, P) < \theta$ (i.e., the alternative hypothesis is true), the power of the E-KSD-Boot test is defined as $\Pr( T_{n,m}^{\KSD, \theta} > \gamma )$, where $\gamma$ is an appropriate critical value controlling the Type-I error rate. The power of the other tests can be defined similarly. We address the following question: Given $\beta \in (0, 1)$ and observed samples $Q_n$ (or $Q_n$ and $P_m$ for two-sample testing), how can we choose the equivalence bound $\theta$ so that we can reject specific alternative distributions with power at least $1-\beta$?

We consider alternatives that are at most $\theta'$ away from $P$ for some $\theta' < \theta$, i.e., $\KSD(Q, P) \leq \theta'$. In this case, the inequality \eqref{eq:KSD_traingle_ineq} implies
\begin{align*}
    T_{n}^{\KSD, \theta}
    \;=\;
    \theta - \KSD(Q_n, P)
    \;&\geq\;
    \theta - \MMD(Q, Q_n; u_p) - \KSD(Q, P)
    \;\geq\;
    \theta - \theta' - \MMD(Q, Q_n; u_p)
    \;,
\end{align*}
where the last inequality holds since $\KSD(Q, P) \leq \theta'$. Writing for simplicity that $\xi_n \coloneqq \MMD(Q, Q_n; u_p)$
\begin{align*}
    \Pr(T_{n}^{\KSD,\theta} \;>\; \gamma_{1-\alpha})
    \;\geq\;
    \Pr(\theta - \theta' - \xi_n > \gamma_{1-\alpha})
    \;=\;
    \Pr(\xi_n < \theta - \theta' - \gamma_{1-\alpha})
    \;.
\end{align*}
In particular, for the RHS to be at least $1 - \beta$, we should choose $\theta$ so that $\theta - \theta' - \gamma_{1-\alpha}  \geq \gamma_{1 - \beta}$, where $\gamma_{1 - \beta}$ is the $(1-\beta)$-quantile of the distribution of $\xi_n$. In practice, such $\gamma_{1 - \beta}$ can be approximated using bootstrap samples as described in \Cref{sec:one_sample:bootstrap_test}. To summarize, choosing $\theta = \theta' + \gamma_{1-\alpha} + \gamma_{1 - \beta}$ guarantees an asymptotic test power of at least $1-\beta$ for all $Q$ with $\MMD(Q, P) \leq \theta'$. The following result shows that, assuming the exact quantiles are used (namely, ignoring bootstrap approximation errors), the test with $\theta$ so selected achieves the desired power if $\KSD(Q, P) \leq \theta'$, while it still controls asymptotic Type-I error if $\KSD(Q, P) > \theta'$. Its proof is in \Cref{thm:power_selected_theta_ksd}.

\begin{theorem}
\label{thm:power_selected_theta_ksd}
    Denote by $\gamma_\rho$ the $\rho$-th quantile of the distribution of $\MMD(Q_n, Q; u_p)$, where $\rho \in (0, 1)$. For any $\theta'$ and $\KSD(Q, P) \leq \theta'$, letting $\theta = \theta' + \gamma_{1-\alpha} + \gamma_{1-\beta}$, we have
    \begin{align}
        \Pr(T_{n}^{\KSD,\theta} > \gamma_{1-\alpha}) \;\geq\; 1 - \beta \;.
        \label{eq:power_selected_power_ksd}
    \end{align}
    Moreover, if instead $\KSD(Q, P) > \theta'$, then
    \begin{align}
        \Pr(T_{n}^{\KSD,\theta} > \gamma_{1-\alpha}) \;\to\; 0 \;.
        \label{eq:power_selected_type_1_ksd}
    \end{align}
\end{theorem}

\begin{remark}
    For example, when $\theta' = 0$, the induced equivalence margin is $\theta = \gamma_{1-\alpha} + \gamma_{1-\beta}$, and \Cref{thm:power_selected_theta_ksd} shows that the power guarantee \eqref{eq:power_selected_power_ksd} holds when $\KSD(Q, P) = 0$, or equivalently $Q = P$, while the calibration \eqref{eq:power_selected_type_1_ksd} holds when $\KSD(Q, P) > 0$, or equivalently $Q \neq P$.
\end{remark}
\begin{remark}
    \Cref{thm:power_selected_theta_ksd} does not account for approximation errors due to the bootstrap step, since it concerns the (unknown) quantiles $\gamma_{1-\alpha}, \gamma_{1-\beta}$ instead of their bootstrap approximation. Consequently, the power guarantee \eqref{eq:power_selected_power_ksd} might not hold with finite samples in practice. However, we find in our experiments in \Cref{sec:experiments} that such power guarantee does hold with finite samples. In fact, it can sometimes be conservative, in the sense that it can yield a larger power by choosing an unnecessarily large $\theta$. We will discuss this further in \Cref{sec:exp:gauss_ms}.
\end{remark}
\begin{remark}
    This practice of choosing $\theta$ based on a desired power has been adopted widely in applications such as psychology studies \citep{lakens2014performing,lakens2017equivalence}, sometimes under the name the ``\emph{small-telescopes}'' approach \citep{simonsohn2015small}. In that context, $\theta$ is also called the \emph{smallest effect size of interest} (SESOI). In some applications such as clinical trials, one might be able to infer an appropriate equivalence bound from pre-set regulations \citep{piaggio2006reporting,meyners2012equivalence,dette2018equivalence}.
\end{remark}

The same idea can be applied to the MMD-based bootstrapping test in a straightforward manner. The only difference is that the critical value should be the $(1-\beta)$-th quantile of $S_{n,m} \coloneqq \MMD(Q_n, Q) + \MMD(P_m, P)$ instead of $\xi_n$. We summarize this in the following result, proven in \Cref{pf:thm:power_selected_theta}.

\begin{theorem}
\label{thm:power_selected_theta}
    For any $\rho \in (0, 1)$, denote by $\eta_\rho$ the $\rho$-th quantile of the distribution of $S_n = \MMD(Q_n, Q) + \MMD(P_n, P)$. For any $\theta'$ and $\MMD(Q, P) \leq \theta'$, letting $\theta = \theta' + \eta_{1-\alpha} + \eta_{1-\beta}$, we have
    \begin{align}
        \Pr(T_{n,m}^{\MMD,\theta} > \eta_{1-\alpha}) \;\geq\; 1 - \beta \;.
        \label{eq:power_selected_power}
    \end{align}
    Moreover, if instead $\MMD(Q, P) > \theta'$, then
    \begin{align}
        \Pr(T_{n,m}^{\MMD,\theta} > \eta_{1-\alpha}) \;\to\; 0 \;.
        \label{eq:power_selected_type_1}
    \end{align}
\end{theorem}

\section{Experiments}
\label{sec:experiments}

We now evaluate the proposed equivalence tests on numerical experiments. We use the IMQ kernel for the KSD-based tests and the RBF kernel for the MMD-based tests; these are standard choices in the literature; see, e.g., \citet{gorham2017measuring,anastasiou2023steins} for KSD and \citet{gretton2012kernel,muandet2017kernel} for MMD. All kernel bandwidths are selected using the median heuristic \citep{gretton2012kernel}, defined as
\begin{align*}
    \lambda_{\mathrm{med}}^2 \;=\; \mathrm{Median}\big\{\|x - y \|_2^2: x, y \in \mathcal{D}, x \neq y \big\} \;,
\end{align*}
where $\mathcal{D} = \{x_i\}_{i=1}^n$ for KSD, and $\mathcal{D} = \{x_i\}_{i=1}^n \cup \{y_i\}_{i=1}^n$ for MMD. Each experiment is repeated 100 times, and we report the proportion of rejections as well as the $95\%$ confidence intervals. All tests have nominal level $\alpha = 0.05$.

\subsection{Gaussian Mean-Shift Models}
\label{sec:exp:gauss_ms}
\begin{figure}
    \centering
    \includegraphics[width=1\linewidth]{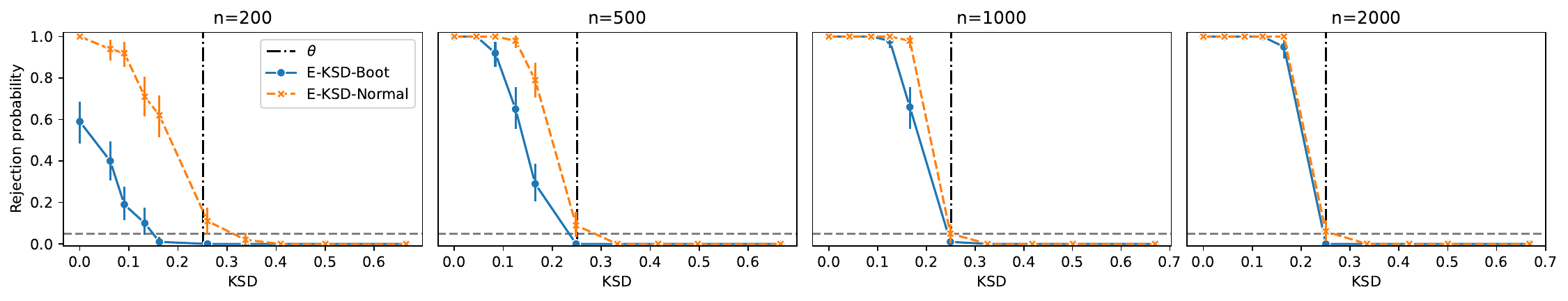}
    \caption{Gaussian mean-shift experiments with varying sample sizes. The black dotted vertical line is the equivalence margin $\theta$.}
    \label{fig:ksd:gauss_ms:sample_size}
\end{figure}

We illustrate the proposed tests with a Gaussian mean-shift example. We draw samples from the standard Gaussian $Q = \cN(0, 1)$, and the nominal distribution is $P = \cN(\mu, 1)$ with some $\mu \in \R$. The score function of such $P$ is $s_p(x) = - x + \mu$. We perform the E-KSD-Normal and E-KSD-Boot tests and compare their Type-I error and power.

\paragraph{Sample size}
We first evaluate the effect of sample sizes in \Cref{fig:ksd:gauss_ms:sample_size}. We fix the equivalence margin $\theta$ to be the population KSD value when $\mu = 0.3$. We can see that both tests correctly controls the Type-I error when the KSD is greater than $\theta$. When the KSD equals $\theta$, namely $Q$ lies at the boundary of the null set, E-KSD-Boot still controls Type-I error, even with a moderate sample size. In comparison, E-KSD-Normal fails to do so, even though both tests are shown to be valid asymptotically by \Cref{thm:ksd_normal_test} and \Cref{thm:ksd_bootstrap_test}. This is because E-KSD-Normal is based on a CLT, which is only a good approximation when $\KSD(Q, P) > 0$. As $\KSD(Q, P)$ tends to 0, the limiting distribution of the KSD estimator is no longer Gaussian, as shown in \Cref{prop:ksd_asymptotics}. Instead, in this case, \eqref{eq:ksd_vstat} becomes degenerate with an asymptotic distribution that is an infinite sum of weighted chi-squares (see the discussion after \Cref{prop:ksd_asymptotics}). Therefore, the normal approximation will have non-negligible errors when  $\KSD(Q, P)$ is close to zero. This poor Type-I error control is also found in other equivalence tests based on asymptotic normality \citep[Section 4]{chen2023testing}. In contrast, E-KSD-Boot does not suffer from this issue, as it does not rely on any normal approximations.

When the KSD is smaller than $\theta$, namely we are under $H_1$, both tests have non-trivial power, with the E-KSD-Normal test having the higher power. Importantly, the power of both tests converge to 1 as the sample sizes tend to infinity, which is expected given their consistency shown in \Cref{thm:ksd_normal_test} and \Cref{thm:ksd_bootstrap_test}.

\begin{figure}
    \centering
    \includegraphics[width=1\linewidth]{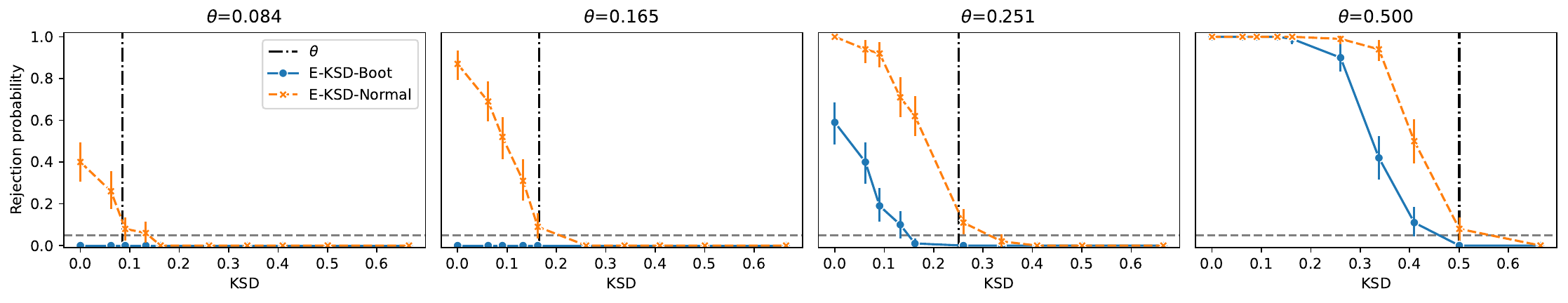}
    \caption{Gaussian mean-shift experiments with varying equivalence margins $\theta$, selected to be the population KSD values for different mean shifts.}
    \label{fig:ksd:gauss_ms:theta}
\end{figure}

\paragraph{Equivalence margin}
Next, we investigate the impact of equivalence margin in \Cref{fig:ksd:gauss_ms:theta}. We fix the sample size to be $n = 200$, and set the equivalence margin $\theta$ to be the population KSD value with $\mu = 0.1, 0.2, 0.3, 0.6$, respectively. For all $\theta$, the E-KSD-Normal test has a higher power, although it has poor Type-I error control near the equivalence margin, especially when $\theta$ is small. As discussed before, this is because the normal approximation deteriorates as $\theta \to 0$. On the other hand, the E-KSD-Boot test has a better Type-I error control at the expense of a lower test power.

\paragraph{Power-selected $\theta$}
We then evaluate the approach proposed in \Cref{sec:choice_of_equiv_margin}, where $\theta$ is selected by a power guarantee with a Type-II error control of $\beta = 0.2$ and $\theta' = 0$. In words, we select the $\theta$ with which the E-KSD-Boot test would achieve a power of at least $1-\beta = 0.8$ when $Q = P$. The results, shown in the top row of \Cref{fig:ksd:power_theta}, confirm that this approach indeed gives the desired power guarantee with E-KSD-Boot for all sample sizes. More importantly, this does not sacrifice the validity, as E-KSD-Boot still has a well-calibrated Type-I error rate. In comparison, E-KSD-Normal fails to control the Type-I error for $Q$ near the boundary of the null set, due to reasons discussed before.

We repeat the same experiment for the MMD-based equivalence tests, E-MMD-Boot and E-MMD-Normal, in the bottom of \Cref{fig:ksd:power_theta}. The same conclusion can be drawn. In particular, with the power-selected $\theta$, the E-MMD-Boot test is well-calibrated when $\MMD(Q, P) \geq \theta$, whereas when $Q = P$, it achieves the desired power of $1 - \beta = 0.8$. However, unlike for the KSD-based bootstrapping test, the power of E-MMD-Boot at $Q = P$ is close to 1, far higher than $1 - \beta$. Intuitively, this is because the $\theta$ derived for the MMD tests rely on a lower bound using \emph{two} triangle inequalities (see the proof of \Cref{thm:power_selected_theta}), whereas the KSD counterpart only uses one. As a result, the power-selected margin for E-MMD-Boot is more conservative, thus the higher power.

\begin{figure}
    \centering
    \includegraphics[width=1\linewidth]{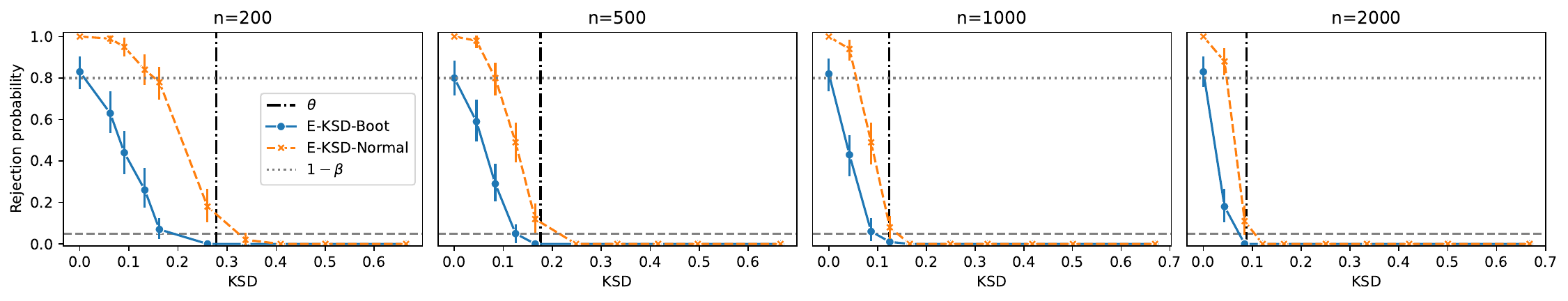}
    \includegraphics[width=1\linewidth]{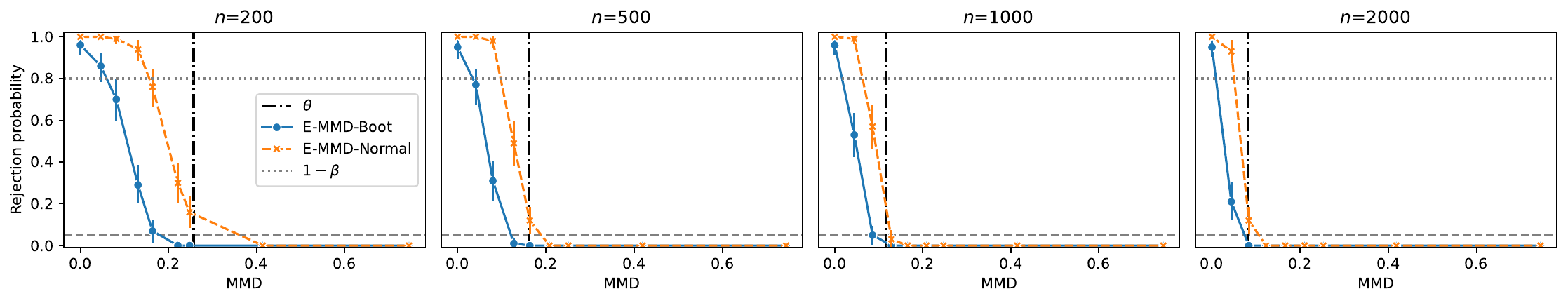}
    \caption{Gaussian mean-shift experiments with $\theta$ selected using a power guarantee with $\beta = 0.2$. \emph{Top}. KSD-based equivalence tests. \emph{Bottom}. MMD-based equivalence tests. }
    \label{fig:ksd:power_theta}
\end{figure}

\subsection{Gaussian-Bernoulli Restricted Boltzmann Machine}
We use the proposed KSD-based equivalence tests to assess the goodness-of-fit of Gaussian-Bernoulli Restricted Boltzmann Machines (GB-RBM) \citep{welling2004exponential,hinton2006reducing}. GB-RBM is a latent-variable model with joint density $p(x, h) \propto \exp(\frac{1}{2} x^\top Bh + b^\top x + c^\top h - \frac{1}{2} \| x \|_2^2)$, where $x \in \R^d$ is a vector of observed variables, $h \in \{\pm 1\}^{d'}$ is a binary hidden variable of latent dimension $d'$, and the model parameters are $B \in \R^{d \times d'}, b \in \R^d, c \in \R^{d'}$. The normalizing constant of a GB-RBM is in general intractable, because marginalizing over the hidden variable requires computing a sum of $2^{d'}$ terms, which is computationally prohibitive when $d'$ is large. Nevertheless, its score function admits a closed form: $s_p(x) = b - x + B \tanh(B^\top x + c)$, where $\tanh$ is applied entry-wise. This makes GB-RBMs a common benchmark for assessing KSD-based GOF tests \citep{liu2016kernelized,jitkrittum2017lineartime,schrab2022ksd,liu2025robustness}. 

We set $d = 50$ and $d' = 10$. The target distribution $P$ is formed with randomly initialized $B, b, c$, where each entry of $b, c$ is drawn independently from $\cN(0, 1)$, and the entries of $B$ are drawn from $\{\pm1\}$ with equal probability. The sampling distribution $Q$ has the same $B, c$, but the parameter $b$ is replaced with the perturbed version $\tilde{b} \sim \cN(b, \sigma^2 I_d)$, for some $\sigma > 0$. We draw samples of size $n = 500$ from $Q$ using block Gibbs sampling as done in \citet{cho2013gaussian,jitkrittum2017lineartime}. We then perform E-KSD-Normal and E-KSD-Boot by setting $\theta$ using a Type-II error control of $\beta = 0.2$ and $\theta' = 0$. 

\Cref{fig:ksd:rbm} shows the results with increasing noise levels $\sigma$. We plot the rejection probability against both the KSD values $\KSD(Q,P)$ and the noise levels on the x-axis; the purpose of the first plot is to demonstrate the calibration and power of the tests by showing $\theta$. Both tests have the correct Type-I error when $\KSD(Q, P) \geq  \theta$. For sufficiently small $\sigma$ such that the alternative hypothesis holds, both tests achieve non-trivial power, with the E-KSD-Normal test being more powerful than E-KSD-Boot. Moreover, at $\sigma = 0$, namely when $Q = P$, E-KSD-Boot has power slightly larger than $1 - \beta = 0.8$. This is expected given that $\theta$ is selected using the smallest effect-size approach. In comparison, E-KSD-Normal achieves a perfect power of 1, suggesting that such $\theta$ selected is conservative for this example. Nevertheless, we note that E-KSD-Normal may fail to control the Type-I error in some cases, as demonstrated in \Cref{sec:exp:gauss_ms}, and thus a direct power comparison between E-KSD-Normal and E-KSD-Boot can be unfair.

\begin{figure}
    \centering
    \includegraphics[width=0.6\linewidth]{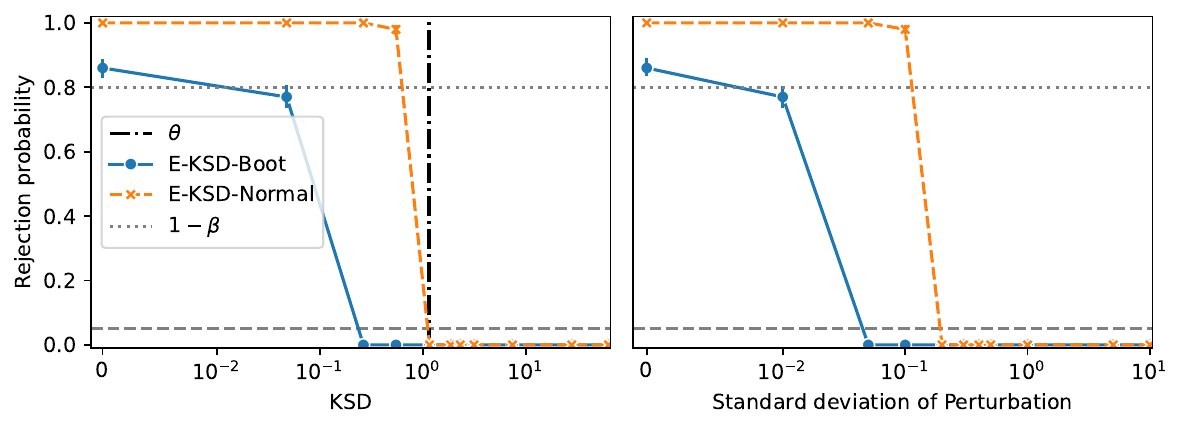}
    \caption{Gaussian-Bernoulli RBM experiment with $\theta$ selected using a Type-II error guarantee with $\beta = 0.2$. The same results are plotted against the KSD values (left) and the standard deviations of the noise (right).}
    \label{fig:ksd:rbm}
\end{figure}

\subsection{MNIST}

\begin{figure}
    \centering
    \includegraphics[width=1\linewidth]{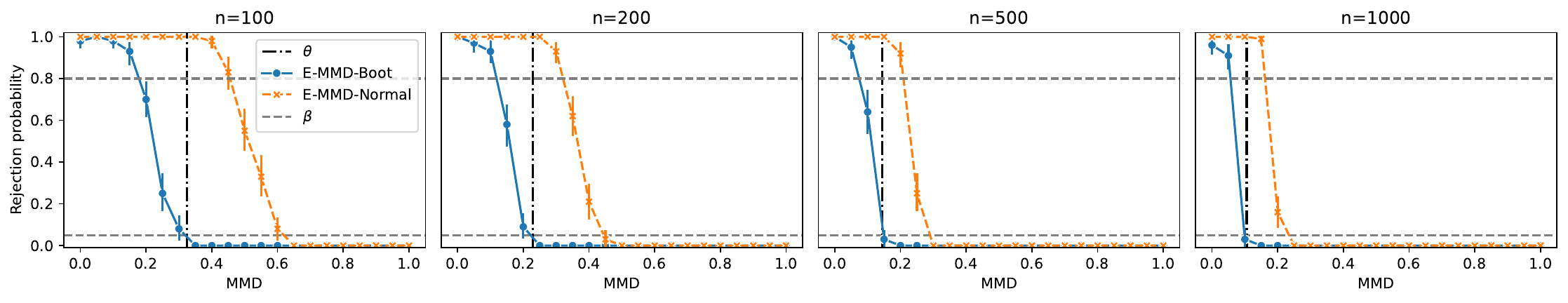}
    \caption{Testing equivalence of MNIST samples. The $\theta$ is selected using a power guarantee with $\beta = 0.2$.}
    \label{fig:mmd:mnist}
\end{figure}

We now evaluate the MMD-based equivalence tests using the MNIST dataset \citep{lecun2010mnist}. Denote by $P_l$ all images of digit $l$, for $l \in \{0, \ldots, 9\}$. The nominal distribution $P = P_1$ is all images of digit $1$, and the sampling distribution $Q = (1-\epsilon) P_1 + \epsilon P_3$ is a mixture of images of digit 1 and 3, with mixing ratio $\epsilon \in [0, 1]$. The purpose of this example is to assess the proposed MMD-based tests in high dimensions and under the presence of outliers. Similarly as before, we select $\theta$ using a Type-II error guarantee of $\beta = 0.2$ and $\theta' = 0$. 

\Cref{fig:mmd:mnist} shows the results as the mixing ratio $\epsilon$ increases from 0 to 1. Under the null hypothesis, where $\MMD(Q, P) \geq \theta$, the E-MMD-Normal test has poor Type-I error control. We again attribute this to the degraded normal approximation due to the high dimensional nature of this example (each image is represented by a vector of 784 dimensions). In contrast, the E-KSD-Boot test is well-calibrated. Moreover, the E-KSD-Boot is not overly conservative, as it achieves non-trivial power for small values of $\MMD(Q, P)$. In particular, when $\MMD(Q, P) = 0$, the realized power is close to 1. This again shows that the smallest effect-size approach is able to select a meaningful margin $\theta$, although in this case it is rather conservative, as the power is higher than the prescribed value of $1 - \beta = 0.8$.

\section{Conclusion}

We propose kernel-based equivalence tests that are suitable for evaluating the equivalence between distributions. Standard goodness-of-fit testing is only designed to assess the \emph{lack} of goodness-of-fit, but cannot draw conclusions on the \emph{equivalence} between observed data and the nominal distribution with probabilistic guarantees on the error rate. Our proposed tests, on the other hand, are shown to be asymptotically valid against null hypotheses that the data and the nominal distribution differ at least by some equivalence margin. Leveraging kernel-based statistical discrepancies, namely the kernel Stein discrepancy and the maximum mean discrepancy, our proposed suit of tests covers both the one-sample and two-sample testing scenarios. Since one challenge with these equivalence tests is the selection of an appropriate equivalence margin, we introduce a data-driven approach where an equivalence margin is selected to be the minimal effect size for achieving a pre-specified power against some certain alternatives.

We conclude with a discussion on some future directions. 
\begin{itemize}
    \item \textbf{Extension to other statistical discrepancies.} Our proof techniques extend naturally to generic one-sample and two-sample V-statistics. It would therefore be interesting to investigate other statistical discrepancies that admit such estimators when constructing similar equivalence tests. Examples include \emph{energy distance} \citep{szekely2004testing} and the \emph{Hilbert-Schmidt independence criterion} \citep[HSIC,][]{gretton2005measuring,gretton2007kernel}, both of which are closely connected to MMD \citep{sejdinovic2013equivalence}.
    \item \textbf{Alternative estimators.} As discussed in \Cref{sec:two_sample:bootstrap_test} and \Cref{sec:one_sample:bootstrap_test}, the normality-based tests can have uncontrolled Type-I error rates when $\theta$ is small, because the CLT approximation of the KSD and MMD estimators deteriorates when $\KSD(Q, P)$ and $\MMD(Q, P)$ are close to zero. One potential solution is to use alternative estimators that retain asymptotic normality even when $\KSD(Q, P)=0$ or $\MMD(Q, P)=0$, such as those based on \emph{incomplete U-statistics} \citep{kim2024dimension,shekhar2022permutation}. Understanding whether these alternative estimators can improve the Type-I error control of these normality-based tests is another possible direction. 
    \item \textbf{Selection of equivalence margin.} Finally, it is important to provide a more intuitive interpretation of the equivalence margin $\theta$ selected using the minimal-effect approach. This is particularly relevant in fields such as bioequivalence studies, where statistical equivalence is often defined through explicit biological criteria \citep{berger1996bioequivalence,dette2021bio,meyners2012equivalence}. Relating these criteria to the value of $\theta$ chosen via the minimal-effect approach in \Cref{sec:choice_of_equiv_margin} may improve the interpretability and practical applicability of kernel-based equivalence tests.
\end{itemize}

\section*{Acknowledgments}
The authors would like to thank Linying Yang for helpful discussions.

% REFERENCE %%%%%%%%%%%%%%%%%%%%%%%%%%%%%%%%%%%
% \bibliographystyle{abbrvnat}
\bibliographystyle{apalike}
\bibliography{ref}

% APPENDIX %%%%%%%%%%%%%%%%%%%%%%%%%%%%%%%%%%%
\newpage
\appendix

\section{Supplementary Background and Results}

We include auxiliary results that will become useful in the proofs in \Cref{sec:proofs}. \begin{itemize}
    \item \Cref{sec:large_sample_asymptotics} studies the weak limit of statistics of the form $\MMD(Q_n, Q)$.
    \item \Cref{sec:normal_equiv_test_general} presents a general framework for the normality-based equivalence tests that applies to both KSD and MMD.
\end{itemize} 

\subsection{Large-Sample Asymptotics}
\label{sec:large_sample_asymptotics}
We present the asymptotic distribution of $\MMD(Q_n, Q)$, where $Q_n$ is an empirical distribution based on random samples from $Q$. This type of statistics appear in the bootstrap validity results \Cref{lem:one_sample_bootstrap_validity} and \Cref{cor: validity of bootstrap}. It is \emph{not} a typical MMD statistic, as the second argument is a population distribution, and thus this statistic is not computable. Nevertheless, we show that it converges to a weak limit at rate $\sqrt{n}$. This is not surprising, since $\MMD^2(Q_n, Q)$ can be shown to be a degenerate (one-sample) V-statistics, and thus converges weakly at rate $n$. Its proof is in \Cref{pf:lem:D_Qn_Q_asymptotics}

\begin{lemma}
\label{lem:D_Qn_Q_asymptotics}
    Let $Q_n$ be an empirical distribution based on an i.i.d.\ random sample $\{X_i\}_{i=1}^n$ drawn from $Q$. Assume $k$ is a reproducing kernel with $\E_{X, X' \sim Q}[|k(X, X')|^2] < \infty$. Then $\sqrt{n} \MMD(Q_n, Q)$ converges weakly to a non-degenerate distribution.
\end{lemma}

\subsection{Normality-based Equivalence Test: A General Framework}
\label{sec:normal_equiv_test_general}
Both \Cref{thm:ksd_normal_test} and \Cref{thm:mmd_normal_test} are kernel equivalence tests based on normal approximations, with the main difference being that the former uses KSD as the statistical divergence, while the latter uses MMD. Nonetheless, the proof for their validity and consistency follows a similar approach. This section presents a general framework for their proofs, which will be tailored to show their individual cases in \Cref{pf:thm:ksd_normal_test} and \Cref{app:proofs:mmd}, respectively.

The proof relies on the following assumption on the asymptotic distributional limit of estimators for the statistical divergences used. It holds for both KSD and MMD.
\begin{assumption}
\label{assump:general_asymptotics}
    Let $\{Q_l\}_{l=1}^\infty$ and $\{P_l\}_{l=1}^\infty$ be two sequences of probability measures. Let $k$ be a reproducing kernel. For each $l$, define $\hat{D}_l \coloneqq \MMD(Q_l, P_l; k)$ and $D \coloneqq \MMD(Q, P; k)$, and assume they are well-defined. Suppose $\{\hat{\sigma}_l\}_{l=1}^\infty$ is a sequence of non-negative numbers such that $\hat{\sigma}_l \to \sigma$ for some $\sigma \geq 0$. Furthermore, assume
    \begin{enumerate}
        \item When $Q \neq P$, then $\sigma > 0$, and $\sqrt{l}(\hat{D}_l - D) \to \cN(0, \sigma^2)$.
        \item When $Q = P$, then $\sigma = 0$, and $\sqrt{l}\hat{D}_l \to 0$ in probability.
    \end{enumerate}
\end{assumption}

We are now ready to present the following result for the validity and consistency of the normality-based test. Its proof is in \Cref{pf:prop:general_normal_test}.
\begin{proposition}
\label{prop:general_normal_test}
    Suppose \Cref{assump:general_asymptotics} holds. Define $\hat{S}^\theta_l =\sqrt{l} (\hat{D}^2_l - \theta^2) / \hat{\sigma}_l$, where $\theta > 0$, and $\hat{\sigma}_l^2 \to \sigma^2$ in probability. Further assume that $\E_{X, X' \sim Q}[|k(X, X')|^2] < \infty $. Then
    \begin{align*}
        \lim_{l \to \infty} \Pr(\hat{S}^\theta_l > \gamma_{1-\alpha})
        \;=\;
        \begin{cases}
            0 \;, &D > \theta \;, \\
            \alpha \;, &D = \theta \;, \\
            1 \;, & D < \theta \;.
        \end{cases}
    \end{align*}
\end{proposition}

\subsection{Intuition behind the Proposed Bootstrapping Approach}
We provide intuition for the choice of the bootstrap sample \eqref{eq: bootstrap sample}. To understand why this is the correct bootstrap statistic, we introduce the concept of \emph{degeneracy} \citep[Chapter 5]{serfling2009approximation}. A symmetric function $u: \R^d \times \R^d \to \R$ (or its corresponding V-statistic) is said to be $Q$-\emph{degenerate} if the function $g(x; Q) \coloneqq \E_{Y \sim Q}[ u(x, Y) ] - \E_{X, Y \sim Q}[u(X, Y)]$ is zero, i.e., $g(x; Q) \equiv 0$, and \emph{non-degenerate} otherwise. Crucially, the asymptotic behaviour of V-statistics varies depending on the degeneracy of $u$, and bootstrapping methods for each case have been well-studied \citep{arcones1992bootstrap,janssen1997bootstrapping}. A natural approach to construct a bootstrap sample for $\MMD^2(Q_n, Q)$ is then to identify it as a V-statistic and use the corresponding bootstrap statistic according to its degeneracy.

Using the closed-form expression for MMD, we can write
\begin{align*}
    \MMD^2(Q_n, Q)
    \;&=\;
    \frac{1}{n^2} \sum_{1 \leq i, j \leq n} \E_{Y, Y' \sim Q}\big[ h\big( (x_i, Y), (x_j, Y') \big) \big]
    \\
    \;&=\;
    \frac{1}{n^2} \sum_{1 \leq i, j \leq n} k(x_i, x_j) - \E_{Y \sim Q}[ k(x_i, Y) ] - \E_{Y' \sim Q}[ k(Y', x_j) ] + \E_{Y, Y' \sim Q}[ k(Y, Y') ]
    \\
    \;&=\;
    \frac{1}{n^2} \sum_{1 \leq i, j \leq n} k(x_i, x_j; Q)
    \;,
\end{align*}
where the second line follows by substituting the definition of the MMD kernel $h\big( (x, y), (x', y') \big) \coloneqq k(x, x') + k(y, y') - k(x, y') - k(x', y)$, and $\E_{Y, Y' \sim Q}[ k(Y, Y') ] = \int_{\R^d} \int_{\R^d} k(x, y) Q(\diff x) Q(\diff y)$, and in the last line we have defined
\begin{align*}
    k(x, y; Q)
    \;\coloneqq\;
    k(x, y) - \E_{Y \sim Q}[ k(x, Y) ] - \E_{X \sim Q}[ k(X, y) ] + \E_{Y, Y' \sim Q}[ k(Y, Y') ]
    \;.
\end{align*}
It is straightforward to see that a generic $k$ is \emph{not} $Q$-degenerate, whereas $k(x, y; Q)$ is. To construct bootstrap samples for V-statistic $\MMD^2(Q_n, Q)$, a naive approach would be to follow the classic construction for degenerate V-statistics \citep{huskova1993consistency}, which would correspond to replacing $k(x_i, x_j)$ in \eqref{eq: bootstrap sample} by $k(x_i, x_j; Q)$. However, this is \emph{not} feasible in practice, as $k(x_i, x_j; Q)$ involves expectations over $Q$, which are intractable.

Instead, we observe that $\MMD^2(Q_n, Q)$ is the second-order term in the \emph{Hoeffding decomposition} \citep[see, e.g., ][Eq. (2.6)]{arcones1992bootstrap} of a non-degenerate V-statistic defined by $n^{-2} \sum_{1 \leq i, j \leq n} k(x_i, x_j)$. Classic bootstrapping results for each term in the Hoeffding decomposition \citep[Lemma 2.3]{arcones1992bootstrap} then suggest that the correct bootstrap statistic for this term is \eqref{eq: bootstrap sample}.

\section{Proofs}
\label{sec:proofs}
This section contains proofs for our theoretical contributions. Proofs relevant to the KSD-based tests are held in \Cref{app:proofs:ksd}, while those relevant to the MMD-based tests are held in \Cref{app:proofs:mmd}.

\subsection{Proofs for KSD Equivalence Tests}
\label{app:proofs:ksd}
This subsection holds proofs for the results on the E-KSD-Normal and E-KSD-Boot tests.

\subsubsection{Proof of \Cref{prop:ksd_asymptotics}}
\label{pf:prop:ksd_asymptotics}
\begin{proof}[Proof of \Cref{prop:ksd_asymptotics}]
    Since \eqref{eq:ksd_vstat} is a V-statistic, classic results on the asymptotics of V-statistics show that \emph{(i)} if $\Var_{X \sim Q}(\E_{X' \sim Q}[ u_p(X, X') ]) > 0$, then $\sqrt{n}(\KSD^2(Q_n, P) - \KSD^2(Q, P)) \darrow \cN(0, \sigma^2_\KSD)$, and \emph{(ii)} if $\Var_{X \sim Q}(\E_{X' \sim Q}[ u_p(X, X') ]) = 0$, then $n \cdot \KSD^2(Q_n, P)$ converges weakly to an infinite sum of weighted chi-squares, so in particular $\sqrt{n} \cdot \KSD(Q_n, P) \to 0$ in probability; see, for example, \citet[Section 6.4.1]{serfling2009approximation}. Therefore, it is sufficient to show that $Q = P$ if and only if $\Var_{X \sim Q}(\E_{X' \sim Q}[u_p(X, X']) = 0$.

    Define $\xi_p(x, \cdot) \coloneqq s_p(x) k(x, \cdot) + \nabla_1 k(x, \cdot)$. It can be shown that the Stein reproducing kernel can be rewritten as $u_p(x, x') = \langle\xi_p(x, \cdot), \xi_p(x', \cdot) \rangle_{\cH_k^d}$ \citep[see, for example, the proof of][Proposition 2]{gorham2017measuring}. First assume $Q = P$. Under \Cref{assump:kernel_universal}, \ref{assump:score_condition} and \ref{assump:ksd_kernel_condition}, \citet[Proposition 1]{gorham2017measuring} shows that $\E_{X' \sim Q}[\xi_p(X', \cdot)] \equiv 0$. This implies that
    \begin{align*}
        \Var_{X \sim Q}(\E_{X' \sim Q}[u_p(X, X')])
        \;&=\;
        \Var_{X \sim Q}(\E_{X' \sim Q}[\langle\xi_p(x, \cdot), \xi_p(x', \cdot) \rangle_{\cH_k^d}])
        \\
        \;&=\;
        \Var_{X \sim Q}(\langle\xi_p(x, \cdot), \E_{X' \sim Q}[\xi_p(x', \cdot)] \rangle_{\cH_k^d})
        \\
        \;&=\;
        0 \;.
    \end{align*}
    We now assume $Q \neq P$ and suppose for a contradiction that $\Var_{X \sim Q}(\E_{X' \sim Q}[u_p(X, X')]) = 0$. Then we must have $\E_{X' \sim Q}[u_p(x, X')] = 0$ for $Q$-almost-sure $x$, and hence $\KSD^2(Q, P) = \E_{X,X' \sim Q}[u_p(X, X')] = 0$. On the other hand, under \Cref{assump:kernel_universal}, \ref{assump:score_condition} and \ref{assump:ksd_kernel_condition}, \Cref{lem:ksd_second_moment} shows that $\E_{X, X' \sim Q}[ |u_p(X, X')|^2 ] < \infty$, and hence \citet[Theorem 3]{barp2024targeted} shows that $\KSD(Q, P) = 0$ if and only if $Q = P$, for all $Q$ with $\E_{X \sim Q}[ \| s_p(X) \|_2 ] < \infty$. We have therefore shown that $Q = P$, a contradiction. Hence, we must have $\Var_{X \sim Q}(\E_{X' \sim Q}[u_p(X, X')]) > 0$.
\end{proof}

\begin{lemma}
\label{lem:ksd_second_moment}
    Suppose \Cref{assump:kernel_universal}, \ref{assump:score_condition} and \ref{assump:ksd_kernel_condition} hold. Then $\E_{X, X' \sim Q}[|u_p(X, X')|^2] < \infty$.
\end{lemma}

\begin{proof}[Proof of \Cref{lem:ksd_second_moment}]
    \citet[Theorem 1]{barp2024targeted} shows that the Stein reproducing kernel $u_p$ is indeed a reproducing kernel. We can hence apply the reproducing property and a Cauchy-Schwarz inequality as done in \citet[Eq.~30]{liu2025robustness} to yield $u_p(x, x') \leq u_p(x, x)^{1/2} u_p(x', x')^{1/2}$. Moreover, we can write $u_p(x, x') = \langle\xi_p(x, \cdot), \xi_p(x', \cdot) \rangle_{\cH_k^d}$, where $\xi_p(x, \cdot) \coloneqq s_p(x) k(x, \cdot) + \nabla_1 k(x, \cdot)$; see, for example, the proof of \citet[Proposition 2]{gorham2017measuring}. We can then further bound $u_p(x, x)$ as
    \begin{align*}
        u_p(x, x)
        \;=\;
        \| \xi_p(x, \cdot) \|_{\cH_k^d}^2
        \;\stackrel{(a)}{\leq}\;
        2\| s_p(x) \|_2^2 + 2\| \nabla_1 k(x, \cdot) \|_{\cH_k^d}^2
        \;=\;
        2\| s_p(x) \|_2^2 + 2\nabla_1^\top \nabla_2 k(x, x)
        \;,
    \end{align*}
    where $(a)$ holds since $(a + b)^2 \leq 2a^2 + 2b^2$ for all constants $a, b$, and in the last step we have used \citet[Lemma 4.34]{steinwart2008support}. Therefore,
    \begin{align*}
        \E_{X, X' \sim Q}[|u_p(X, X')|^2]
        \;&\leq\;
        \big(\E_{X \sim Q}[| u_p(X, X) |^2 ] \big)^2
        \\
        \;&\stackrel{(b)}{\leq}\;
        \big( \E_{X \sim Q}\big[ \big(2\| s_p(X) \|_2^2 + 2\nabla_1^\top \nabla_2 k(X, X) \big)^2 \big] \big)^2
        \\
        \;&\leq\;
        \big( \E_{X \sim Q}\big[ 8\| s_p(X) \|_2^4 + 8 \big(\nabla_1^\top \nabla_2 k(X, X) \big)^2 \big] \big)^2
        \\
        \;&=\;
        \big( 8\E_{X \sim Q}\big[ \| s_p(X) \|_2^4 + 8 \E_{X \sim Q}\big[ \big(\nabla_1^\top \nabla_2 k(X, X) \big)^2 \big] \big)^2
        \;,
    \end{align*}
    where in $(b)$ we have again used the inequality $(a+b)^2 \leq 2a^2 + 2b^2$. Finally, the RHS of the last line is finite under \Cref{assump:kernel_universal}, \ref{assump:score_condition} and \ref{assump:ksd_kernel_condition}, thus finishing the proof.
\end{proof}

\subsubsection{Proof of \Cref{prop:general_normal_test}}
\label{pf:prop:general_normal_test}
First assume $Q \neq P$. Under \Cref{assump:general_asymptotics}, using the weak consistency $\hat{\sigma}_l^2 \to \sigma^2_l$ and Slutsky Lemma gives
\begin{align*}
    \frac{\sqrt{l}}{\hat{\sigma}_l}\big(\hat{D}^2 - D^2 \big)
    \;\darrow\;
    \cN(0, 1)
    \;.
\end{align*}
The probability of rejecting $H_0$ is
\begin{align*}
    \Pr(\hat{S}^{\theta}_l > z_\alpha)
    \;&=\;
    \Pr\bigg(\frac{\sqrt{l}}{\hat{\sigma}_l}(\hat{D}^2 - D^2) > z_\alpha + \frac{\sqrt{l}}{\hat{\sigma}_l}(\theta^2 - D^2) \bigg)
    % \\
    \;\to\;
    \begin{cases}
        0, \qquad &D > \theta \;,\\
        \alpha, \qquad &D = 0 \;,\\
        1, \qquad &D < \theta \;.
    \end{cases}
    \;,
\end{align*}
where the last line holds because $\hat{\sigma}_l \to \sigma > 0$, and thus
\begin{align*}
    \frac{\sqrt{l}}{\hat{\sigma}_l}(\theta^2 - D^2)
    \;\to\;
    \begin{cases}
        \infty, \qquad &D > \theta \;,\\
        0, \qquad &D = 0 \;,\\
        -\infty, \qquad &D < \theta \;.
    \end{cases}
\end{align*}
It then remains to show that $ \Pr(\hat{S}^{\theta}_l > z_\alpha) \to 0$ when $Q = P$. In this case,
\begin{align*}
    \Pr(\hat{S}^{\theta}_l > z_\alpha)
    \;=\;
    \Pr\bigg(\frac{\sqrt{l}}{\hat{\sigma}_l}\hat{D}_l > z_\alpha + \frac{\sqrt{l}}{\hat{\sigma}_l}\theta^2 \bigg)
    \;&=\;
    \Pr\big(\sqrt{l} \hat{D}^2_l > z_\alpha \hat{\sigma}_l + \theta^2\sqrt{l} \big)
    \;\to\;
    0
    \;,
\end{align*}
where the convergence holds because $\sqrt{l} \hat{D}_l^2 \to 0$ and $\hat{\sigma}^2 \to 0$ by \Cref{assump:general_asymptotics}, whereas $\theta^2 \sqrt{l}$ diverges to $+\infty$ since $\theta > 0$.

\subsubsection{Proof of \Cref{thm:ksd_normal_test}}
\label{pf:thm:ksd_normal_test}
\begin{proof}[Proof of \Cref{thm:ksd_normal_test}]
    It suffices to verify that the conditions in \Cref{prop:general_normal_test} for KSD. We consider the sequences $\{Q_l\}_{l=1}^\infty$, where each $Q_l$ is an empirical measure based on independent samples $\{X_i\}_{i=1}^l$ from $Q$, and $P_l = P$, and choose $k = u_p$ to be the reproducing Stein kernel so that $\hat{D}_l = \MMD(Q_l, P; u_p) = \KSD(Q_l, P)$ and $D = \MMD(Q, P; u_p) = \KSD(Q, P)$. In particular, $\KSD(Q, P)$ is well-defined under the assumed moment conditions. Moreover, choose $\hat{\sigma}_l = \hat{\sigma}_\KSD$ as defined in \eqref{eq:ksd_jackknife}.

    Under \Cref{assump:kernel_universal}, \ref{assump:score_condition} and \ref{assump:ksd_kernel_condition}, \Cref{lem:ksd_second_moment} shows that $\E_{X, X' \sim Q}[ |u_p(X, X')|^2 ] < \infty$. Therefore, the asymptotics condition \Cref{assump:general_asymptotics}, which is assumed in \Cref{prop:ksd_asymptotics}, holds for $\hat{D}_l = \KSD(Q_l, P)$, where in particular $\sigma_\KSD > 0$ when $Q \neq P$ and $\sigma_\KSD = 0$ when $Q = P$. It then remains to show that $\hat{\sigma}_\KSD \to \sigma_\KSD$ in probability. Again since $\E_{X, X' \sim Q}[ |u_p(X, X')|^2 ] < \infty$, we can invoke \citet[Theorem 6]{arvesen1969jackknifing} to conclude that the Jackknife estimator of the variance converges weakly, namely $\hat{\sigma}_\KSD^2 \to \sigma_\KSD^2$ in probability.
\end{proof}

\subsubsection{Proof of \Cref{lem:one_sample_bootstrap_validity}}
\label{pf:lem:one_sample_bootstrap_validity}
We will show the result for a U-statistic counterpart and use the fact that U- and V-statistics have similar asymptotic behaviour to conclude the proof. In particular, we will show that $n D^2_{W_n}(\X_n; k) $ and $n \MMD^2(Q_n, Q; k)$ converges weakly to the same limit, which is sufficient to prove the desired result by applying the Mann–Wald Theorem \citep[also known as Continuous Mapping Theorem; cf.][Theorem 2.3]{vandervaart2000asymptotic} and noting that the function $x \mapsto \sqrt{x}$ is everywhere continuous on $[0, \infty)$.

For each $n$, let $X_{n1}^\ast, X_{n2}^\ast, \ldots, X_{nn}^\ast$ denote independent draws from $Q_n$ conditional on $\{X_i\}_{i=1}^\infty$. Given a probability measure $R$, define the symmetric function
\begin{align*}
    k(x, x'; R) \;=\; k(x, x') + \E_{Z, Z' \sim R}[k(Z, Z')] - \E_{Z' \sim R}[k(x, Z')] - \E_{Z \sim R}[k(Z, x')]
    \;.
\end{align*}
We begin by writing the bootstrap sample as
\begin{align*}
    \MMD^2_{W_n}
    \;&=\;
    \frac{1}{n^2} \sum_{1 \leq i, j \leq n} (W_{ni} - 1) (W_{nj} - 1) k(x_i, x_j)
    \\
    \;&=\;
    \frac{1}{n^2} \sum_{1 \leq i, j \leq n} W_{ni} W_{nj} k(x_i, x_j) - W_{nj} k(x_i, x_j) - W_{ni} k(x_i, x_j) + k(x_i, x_j)
    \\
    \;&\stackrel{d}{=}\;
    \frac{1}{n^2} \sum_{1 \leq i, j \leq n} k(X_{ni}^\ast, X_{nj}^\ast) - k(x_i, X_{nj}^\ast) - k(X_{ni}^\ast, x_j) + k(x_i, x_j)
    \\
    \;&=\;
    \frac{1}{n^2} \sum_{1 \leq i, j \leq n} k(X_{ni}^\ast, X_{nj}^\ast; Q_n)
    \\
    \;&=\;
    \frac{n-1}{n} U_n^\ast + R_n
    \;,
\end{align*}
where the third line should be understood as quality in distribution, and we have defined
\begin{align*}
    U_n^\ast \;=\; \frac{1}{n(n-1)} \sum_{1 \leq i \neq j \leq n} k(X_{ni}^\ast, X_{nj}^\ast; Q_n)
    \;,\qquad
    R_n \;\coloneqq\; \frac{1}{n^2} \sum_{ i = 1}^n k(X_{ni}, X_{ni}; Q_n)
    \;.
\end{align*}
Since $k$ is bounded by assumption, with Jensen's inequality, it is straightforward to see that $R_n \to 0$ in probability, and so $D^2_{W_n}$ and $U_n^\ast$ have the same asymptotic distribution (if it exists) by Slutsky's theorem. It then suffices to show the claimed result holds for $U_n^\ast$ in place of $D^2_{W_n}$. To proceed, we define the following U-statistic
\begin{align*}
    U_n \;\coloneqq\; \frac{1}{n(n-1)} \sum_{1 \leq i \neq j \leq n} k(x_i, x_j)
    \;,
\end{align*}
and define $\mu_Q \coloneqq \E_{X, Y \sim Q}[k(X, Y)]$ and $g_k(\cdot; Q) \coloneqq \E_{Y \sim Q}[ k(\cdot, Y) ] - \mu_Q$. Direct computation gives
\begin{align*}
    U_n
    \;\coloneqq\;
    \mu_Q + \frac{2}{n} \sum_{i=1}^n g_k(x_i, Q) + \frac{1}{n(n-1)} \sum_{1 \leq i \neq j \leq n} k(x_i, x_j; Q)
    \;\eqqcolon\;
    \mu_Q + T_{1n} + T_{2n}
    \;,
\end{align*}
which is known as the \emph{Hoeffding's decomposition} of (not necessarily degenerate) one-sample U-statistics; see, e.g., \citet[Eq.~2.6]{arcones1992bootstrap}. Notably, the last term $T_{2n}$ in the Hoeffding's decomposition is the U-statistic counterpart of $\MMD^2(Q_n, Q)$, and is degenerate, so in particular the asymptotic distribution of $n T_{2n}$ exists \citep[Theorem 5.5.2]{serfling2009approximation}. Under the assumed condition $\E_{X, X' \sim Q}[| k(X, X')|^2] < \infty$, results on the bootstrap statistic for each term in the Hoeffding's decomposition \citep[Lemma 2.3, applied with $r = m = 2$]{arcones1992bootstrap} show that a correct bootstrap statistic for $T_{2n}$ is $U_n^\ast$, i.e.,
\begin{align*}
    \sup_{t \in \R} \big|
        \Pr( n U_n^\ast \leq t \;|\; \X_n )
        - \Pr( n T_{2n} \leq t )
    \big| \;\to\; 0
    \;.
\end{align*}
% Finally, under the moment condition $\E_{X, X' \sim Q}[| k(X, X')|^2] < \infty$ and noting that $T_{2n}$ is the U-statistic counterpart of the V-statistic $D_{W_n}^2$, \citet[Theorem 5.1]{grams1973convergence} shows that
% \begin{align*}
%     \E_{X, X' \sim Q}[ | T_{2n} - D_{W_n} ]
% \end{align*}
Moreover, $T_{2n}$ is the U-statistic counterpart of $\MMD^2(Q_n, Q)$, which is obvious by noting that
\begin{align*}
    \MMD^2(Q_n, Q)
    \;&=\;
    \frac{1}{n^2}\sum_{i=1}^n\sum_{j=1}^n k(x_i, x_j) + \E_{X, X' \sim Q}[k(X, X')] - \frac{1}{n}\sum_{i=1}^n \E_{X' \sim Q}[k(x_i, X')] - \frac{1}{n}\sum_{j=1}^n \E_{X \sim Q}[k(X, x_j)]
    \\
    \;&=\;
    \frac{1}{n^2}\sum_{i=1}^n\sum_{j=1}^n k(x_i, x_j; Q)
    ;.
\end{align*}
Under the moment condition $\E_{X, X' \sim Q}[| k(X, X')|^2] < \infty$, \citet[Theorem 5.1]{grams1973convergence} shows that
\begin{align*}
    \E_{X, X' \sim Q}[ | T_{2n} - \MMD^2(Q_n, Q) |^2 ]
    \;=\;
    \cO(n^{-2})
    \;.
\end{align*}
This implies that $T_{2n}$ and $\MMD^2(Q_n, Q)$ converges in probability, thus also in distribution, to the same limit, i.e.,
\begin{align*}
    \sup_{t \in \R} \big|
        \Pr( n \MMD^2(Q_n, Q) \leq t )
        - \Pr( n T_{2n} \leq t )
    \big| \;\to\; 0
    \;.
\end{align*}
Combining these observations and using a triangle inequality, we have shown that
\begin{align*}
    \sup_{t \in \R} \big|
        \Pr( n D_{W_n}^2(\X_n; k) \leq t | \X_\infty)
        - \Pr( n \MMD^2(Q_n, Q) \leq t )
    \big| \;\to\; 0
    \;,
\end{align*}
which completes the proof by applying the Mann-Wald Theorem.

\subsubsection{Proof of \Cref{lem:D_Qn_Q_asymptotics}}
\label{pf:lem:D_Qn_Q_asymptotics}
First, the assumption $\E_{X, X' \sim Q}[|k(X, X')|^2] < \infty$ implies
\begin{align}
    \Var_{X, X' \sim Q}(k(X, X'))
    \;\leq\;
    \E_{X, X' \sim Q}[|k(X, X')|^2]
    \;<\;
    \infty
    \;.
    \label{eq:mmd_one_sample_moment}
\end{align}
We first write
\begin{align*}
    \MMD^2(Q_n, Q)
    \;&=\;
    \frac{1}{n^2} \sum_{1 \leq i, j \leq n} \E_{Y, Y' \sim Q}\big[ h\big( (x_i, Y), (x_j, Y') \big) \big]
    \\
    \;&=\;
    \frac{1}{n^2} \sum_{1 \leq i, j \leq n} k(x_i, x_j) - \E_{Y \sim Q}[ k(x_i, Y) ] - \E_{Y' \sim Q}[ k(Y', x_j) ] + \E_{Y, Y' \sim Q}[ k(Y, Y') ]
    \\
    \;&=\;
    \frac{1}{n^2} \sum_{1 \leq i, j \leq n} k(x_i, x_j; Q)
    \;,
\end{align*}
where the second line follows by substituting the definition of the MMD kernel $h\big( (x, y), (x', y') \big) \coloneqq k(x, x') + k(y, y') - k(x, y') - k(x', y)$, and in the last line we have defined
\begin{align*}
    k(x, x'; Q)
    \;\coloneqq\;
    k(x, x') - \E_{Y \sim Q}[ k(x, Y) ] - \E_{X \sim Q}[ k(X, x') ] + \E_{Y, Y' \sim Q}[ k(Y, Y') ]
    \;.
\end{align*}
It is easy to verify that the function $(x, x') \mapsto k(x, x'; Q)$ is symmetric and \emph{$Q$-degenerate}, meaning that $\E_{Y \sim Q}[k(x, Y; Q)] = 0$ for all $Q$-almost sure $x$ \citep[see, e.g.,][Chapter 5]{serfling2009approximation}. Therefore, $\MMD^2(Q_n, Q)$ is a degenerate \emph{one-sample} V-statistic (not to be confused with degenerate distributions). With this observation and the moment condition \eqref{eq:mmd_one_sample_moment}, we can apply the asymptotics of degenerate one-sample V-statistics \citep[Theorem 6.4.1 B]{serfling2009approximation} to show that, as $n \to \infty$, the random variable $n\MMD^2(Q_n, Q)$ converges weakly to an infinite weighted sum of Chi-square distributions, which is non-degenerate. Moreover, since the square-root function $x \mapsto \sqrt{x}$ is continuous on $[0, \infty)$ and continuous functions preserve weak limits by the Continuous Mapping Theorem \citep[Theorem 2.3]{vandervaart2000asymptotic}, the statistic $\sqrt{n}\MMD(Q_n, Q)$ also converges weakly to a non-degenerate distribution.

\subsubsection{Proof of \Cref{thm:ksd_bootstrap_test}}
\label{pf:thm:ksd_bootstrap_test}
We will show that the claimed result holds with $\gamma_{1-\alpha}^\infty$ replaced by the $(1-\alpha)$-th quantile, $\gamma_{1-\alpha}$, of the (unconditional) distribution of $\MMD(Q_n, Q; u_p)$. This will be sufficient to show the claimed result, since the following bootstrap error vanishes by \Cref{lem:one_sample_bootstrap_validity}
\begin{align*}
    \big| \Pr\big(T_n^{\KSD, \theta} > \gamma_{1-\alpha}^\infty \;|\; \X_n \big) - \Pr\big(T_n^{\KSD, \theta} > \gamma_{1-\alpha} \big) \big|
    \;\to\; 0
    \;.
\end{align*}

We first prove the case \emph{(i)}, namely when $\KSD(Q, P) > \theta$. We first show the inequality \eqref{eq:KSD_traingle_ineq}, restated below
\begin{align*}
    \KSD(Q, P)
    \;=\;
    \MMD(Q, P; u_p)
    \;&\leq\;
    \MMD(Q, Q_n; u_p) + \MMD(Q_n, P; u_p)
    \\
    \;&=\;
    \MMD(Q, Q_n; u_p) + \KSD(Q_n, P)
    \tagaligneq
    \label{eq:KSD_traingle_ineq_pf}
    \;.
\end{align*}
The first step holds by rewriting KSD as an MMD \citep[Theorem 1]{barp2024targeted}, and the inequality follows from the triangle inequality of MMD \citep[see, e.g.,][Lemma 4]{gretton2012kernel}. 
Using this inequality, we can bound the probability of rejecting $H_0: \KSD(Q, P) \geq \theta$ as
\begin{align}
    \Pr\big(T_n^{\KSD, \theta} > \gamma_{1-\alpha} \big)
    \;&=\;
    \Pr\big(\theta - \KSD(Q_n, P) > \gamma_{1-\alpha} \big)
    \notag
    \\
    \;&=\;
    \Pr\big(\KSD(Q, P) - \KSD(Q_n, P) > \gamma_{1-\alpha} + \KSD(Q, P) - \theta \big)
    \notag
    \\
    \;&\leq\;
    \Pr\big(\MMD(Q_n, Q; u_p) > \gamma_{1-\alpha} + \KSD(Q, P) - \theta \big)
    \notag
    \\
    \;&=\;
    \Pr\big(\sqrt{n} \MMD(Q_n, Q; u_p) > \sqrt{n}\gamma_{1-\alpha} + \sqrt{n}(\KSD(Q, P) - \theta) \big)
    \;,
    \label{eq:ksd_test_rej_rewrite}
\end{align}
By \Cref{lem:D_Qn_Q_asymptotics}, the term $\sqrt{n} \MMD(Q_n, Q; u_p)$ converges weakly to a non-degenerate distribution. In particular, its scaled quantile $\sqrt{n} \gamma_{1-\alpha}$ also converges to some real number \citep[Lemma 21.2]{vaart1998asymptotic}. On the other hand, the term $\sqrt{n}(\KSD(Q, P) - \theta)$ diverges to $\infty$ when $\KSD(Q, P) > \theta$. Combining these arguments, we conclude that, when $\KSD(Q, P) > \theta$,
\begin{align*}
    \lim_{n \to \infty}
    \Pr\big(\sqrt{n} \MMD(Q_n, Q; u_p) > \sqrt{n}\gamma_{1-\alpha} + \sqrt{n}(\KSD(Q, P) - \theta) \big)
    \;=\;
    0
    \;.
\end{align*}
We now consider the case \emph{(ii)}, namely $\KSD(Q, P) = \theta$. In particular, this implies that $Q \neq P$, and so the following limit holds by \Cref{prop:ksd_asymptotics}
\begin{align*}
    \sqrt{n} (\KSD(Q_n, P) - \KSD(Q, P)) \;=\; \sqrt{n} (\KSD(Q_n, P) - \theta) \;\to\; \cN(0, \sigma^2_\KSD) \;.
\end{align*}
Moreover, as argued before, $\sqrt{n} \gamma_{1-\alpha} \to c$ for some constant $c$. We then have
\begin{align*}
    \Pr\big(\theta - \KSD(Q_n, P) > \gamma_{1-\alpha} \big)
    \;=\;
    \Pr\big(\sqrt{n}(\KSD(Q_n, P) - \theta) < -\sqrt{n}\gamma_{1-\alpha} \big)
    \;\to\;
    \Phi\Big(-\frac{c}{\sigma_\KSD}\Big) 
    \;\eqqcolon\;
    a
    \;,
\end{align*}
where $\Phi$ denotes the CDF of the standard normal distribution. In particular, we must have $a > 0$, as $-c / \sigma_\KSD$ is a real number. 
% Since $\sqrt{n} \MMD(Q_n, Q; u_p)$ converges weakly and $\sqrt{n} \gamma_{1-\alpha}$ also converges to some real number, there must exists a constant $a \in [0, 1]$ such that
% \begin{align*}
%     \Pr\big(\theta - \KSD(Q_n, P) > \gamma_{1-\alpha} \big)
%     \;&=\;
%     \Pr\big(\sqrt{n} (\KSD(Q_n, P) - \KSD(Q, P)) < \sqrt{n} \gamma_{1-\alpha} + \sqrt{n}(\theta - \MMD(Q, P)) \big)
%     \\
%     \;&=\;
%     \Pr\big(\sqrt{n} (\KSD(Q_n, P) - \KSD(Q, P)) > \sqrt{n}\gamma_{1-\alpha} \big)
%     \\
%     \;&\to\;
%     a \;.
% \end{align*}
We claim that $a \leq \alpha$. Applying the inequality \eqref{eq:KSD_traingle_ineq_pf} again, we have the following bound on the probability of rejecting $H_0$
\begin{align*}
    \Pr\big(\theta - \KSD(Q_n, P) > \gamma_{1-\alpha} \big)
    \;&\leq\;
    \Pr\big(\theta + \MMD(Q_n, Q; u_p) - \KSD(Q, P) > \gamma_{1-\alpha} \big)
    \\
    \;&=\;
    \Pr\big(\MMD(Q_n, Q; u_p) > \gamma_{1-\alpha} + \KSD(Q, P) - \theta \big)
    \\
    \;&=\;
    \Pr\big(\MMD(Q_n, Q; u_p) > \gamma_{1-\alpha} \big)
    \\
    \;&=\;
    \alpha
    \;.
    % \label{eq:rej_prob_bound}
\end{align*}
Taking the limit $n \to \infty$ on both sides then implies $a \leq \alpha$, and so $a \in (0, \alpha]$. This shows \emph{(ii)}. 

To prove the case \emph{(iii)}, we write
\begin{align*}
    &\;
    \Pr\big(\theta - \KSD(Q_n, P) > \gamma_{1-\alpha} \big)
    \\
    \;&=\;
    \Pr\big(\KSD(Q_n, P) < \theta - \gamma_{1-\alpha} \big)
    \\
    \;&=\;
    \Pr\big(\KSD^2(Q_n, P) < (\theta - \gamma_{1-\alpha})^2 \big)
    \tagaligneq
    \label{eq:ksd_power_step1}
    \\
    \;&=\;
    \Pr\big(\sqrt{n}\big(\KSD^2(Q_n, P) - \KSD^2(Q, P) \big) < \sqrt{n}(\theta - \gamma_{1-\alpha})^2 - \sqrt{n}\KSD^2(Q, P) \big)
    \tagaligneq
    \label{eq:ksd_power}
    \;.
\end{align*}
Assuming first that $0 < \KSD(Q, P) < \theta$, then $\E_{X \sim Q}[u_p(X, \cdot)] \not\equiv 0$, and thus $\KSD^2(Q_n, P)$ is a non-degenerate V-statistic; see, e.g., the proof of \citet[Theorem 4.1]{liu2016kernelized} for a proof for U-statistics; extensions to V-statistics is straightforward. Asymptotics of non-degenerate V-statistics \citep[Theorem 6.4.1 B]{serfling2009approximation} then gives $\sqrt{n}\big(\KSD^2(Q_n, P) - \KSD^2(Q, P) \big) \to \cN(0, \sigma_\KSD^2)$ in distribution, where $\sigma_\KSD^2 > 0$. On the other hand, the RHS of the inequality within the probability in \eqref{eq:ksd_power} can be simplified as
\begin{align*}
    \sqrt{n}(\theta - \gamma_{1-\alpha})^2 - \sqrt{n}\KSD^2(Q, P)
    \;=\;
    \sqrt{n}(\theta^2 - \KSD^2(Q, P)) + \sqrt{n} \gamma_{1-\alpha}^2 - 2 \sqrt{n}\theta \gamma_{1-\alpha}
    \;.
\end{align*}
By \Cref{lem:D_Qn_Q_asymptotics}, the scaled quantile $\sqrt{n} \theta \gamma_{1-\alpha}$ converges to a positive constant, while $\sqrt{n} \gamma_{1 - \alpha}^2 \to 0$. Since also $\KSD(Q, P) < \theta$, the RHS of the above equality diverges to $\infty$. We have therefore showed that, inside the probability of \eqref{eq:ksd_power}, the LHS converges weakly to a Gaussian limit, while the RHS diverges to $\infty$. This implies that $\Pr\big(\theta - \KSD(Q_n, P) > \gamma_{1-\alpha} \big) \to 1$.

It then remains to prove \emph{(iii)} when $\KSD(Q, P) = 0$. In this case, $n \KSD^2(Q_n, P)$ converges weakly to a non-degenerate distribution by \Cref{prop:ksd_asymptotics}. In particular, $\KSD^2(Q_n, P) \to 0$ in probability. On the other hand, $\gamma_{1-\alpha} \to 0$ in probability by \Cref{lem:D_Qn_Q_asymptotics}, and so $(\theta - \gamma_{1-\alpha})^2 \to \theta > 0$. In other words, inside the probability in \eqref{eq:ksd_power_step1}, the LHS converges to $0$ in probability, while the RHS converges to a non-zero constant. This implies that \eqref{eq:ksd_power_step1} converges to $1$, thus showing the case \emph{(iii)} when $\KSD(Q, P) = 0$.

\subsection{Proofs for MMD Equivalence Tests}
\label{app:proofs:mmd}
This subsection holds proofs for the results on the E-MMD-Normal and E-MMD-Boot tests.

\subsubsection{Proof of \Cref{prop:mmd_asymptotics}}
\label{pf:prop:mmd_asymptotics}

This is a well-known result and a consequence of classic asymptotic theory for two-sample U-statistics; see, e.g., \citet{huang2023weighted}. For completeness, we provide a proof for the first limit in \Cref{prop:mmd_asymptotics} when $Q \neq P$. Convergence in the other case, namely when $Q = P$, follows directly from \citet[Theorem 12]{gretton2012kernel}, which showed that, in this case, $N \cdot \MMD(Q_n, P_m)$ converges to a infinite sum of weighted chi-squares.
 
When $Q \neq P$, we suppose for a contradiction that $\sigma_1 = \Var_{X \sim Q}(\E_{X' \sim Q, Y, Y' \sim P}[h(X, X', Y, Y')]) = 0$. Then we must have $\E_{X' \sim Q, Y, Y' \sim P}[h(\cdot, X', Y, Y')] = 0$ almost surely. Taking expectation over $Q$ gives $0 = \E_{X, X' \sim Q, Y, Y' \sim P}[h(X, X', Y, Y')] = \MMD(Q, P)$, and hence $Q = P$, a contradiction. We thus conclude that $\sigma_1 > 0$. \citet[Proposition 1]{huang2023weighted} then imply that the first asymptotic limit in \Cref{prop:mmd_asymptotics} holds.

\subsubsection{Proof of \Cref{thm:mmd_normal_test}}
\label{pf:thm:mmd_normal_test}
We again verify the assumptions in \Cref{prop:general_normal_test} for MMD, which is sufficient to conclude the claimed results. For notational convenience, we define the sample sizes $n = n_l$ and $m = m_l$ to be functions of $l$. Let $\{X_i\}_{i=1}^\infty$ and $\{Y_j\}_{j=1}^\infty$ be independent draws from $Q$ and $P$, respectively. We consider the sequences $\{Q_l\}_{l=1}^\infty$ and $\{P_l\}_{l=1}^\infty$, where $Q_l, P_l$ are empirical measures based on $\{X_i\}_{i=1}^{n_l}$ and $\{Y_j\}_{j=1}^{m_l}$, respectively. Moreover, choose $\hat{\sigma}_l = \hat{\sigma}_\MMD$ as defined below \eqref{eq:mmd_normal_test}.

Under \Cref{assump:mmd_sample_sizes} and \Cref{assump:mmd_moment}, \Cref{prop:mmd_asymptotics} shows that the assumed asymptotics hold for $\hat{D}_l = \MMD(Q_l, P_l; k)$, where in particular $\sigma_\MMD > 0$ when $Q \neq P$ and $\sigma_\MMD = 0$ when $Q = P$. It then remains to show that $\hat{\sigma}_\MMD \to \sigma_\MMD$ in probability. Under the assumed condition $\E_{X, X' \sim Q}[ |k(X, X')|^2 ] < \infty$, we can invoke \citet[Theorem 6]{arvesen1969jackknifing} to conclude that the Jackknife estimator of the variance converges weakly, namely $\hat{\sigma}_\MMD^2 \to \sigma_\MMD^2$ in probability.

\subsubsection{Proof of \Cref{cor: validity of bootstrap}}
\label{pf:cor: validity of bootstrap}
As argued in the proof for \Cref{lem:one_sample_bootstrap_validity}, the conditional distribution of $U_n^\ast \coloneqq \sqrt{N} D_{W_n}(\X_n) = \sqrt{n} D_{W_n}(\X_n)$ given $\X_\infty$ and the unconditional distribution of $U_n \coloneqq \sqrt{n} \MMD(Q_n, Q)$ converge weakly to the same distribution, say $R$. That is, given a random variable $U \sim R$, we have $U_n^\ast \to U$ and $U_n \to U$ weakly. Similarly, denote by $R'$ the limiting distribution of $V_n^\ast \coloneqq \sqrt{n} D_{W_n}(\Y_n)$ given $\Y_\infty$ and of $V_n \coloneqq \sqrt{n} \MMD(P_n, P)$, and let $V \sim R'$. We then have $V_n^\ast \to V$ and $V_n \to V$. It suffices to show that $U_n^\ast + V_n^\ast \to U + V$ and $U_n + V_n \to U + V$ weakly and apply a triangle inequality to conclude the claimed result.

For a (real-valued) random variable $Z$, denote by $\phi_Z(t) \coloneqq \E[ \exp(it Z)]$ its characteristic function, where $i$ denotes the imaginary unit. In particular, $\phi_U(t) = \E[\exp(it U)]$ and $\phi_{U_n^\ast}(t) = \E[ \exp(it U_n^\ast) \;|\; \X_\infty ]$. The weak convergence of $U_n^\ast$ and $V_n^\ast$ then implies $\phi_{U_n^\ast}(t) \to \phi_U(t)$ and $\phi_{V_n^\ast}(t) \to \phi_V(t)$ pointwise in $t$. Hence, the product $\phi_{U_n^\ast}(t) \phi_{V_n^\ast}(t) \to \phi_U(t) \phi_V(t)$ also converges pointwise by the independence of $U$ and $V$. Since by assumption $U_n^\ast$ and $V_n^\ast$ are independent given $\X_n$ and $\Y_n$, this implies $\phi_{U_n^\ast}(t) \phi_{V_n^\ast}(t) = \phi_{U_n^\ast + V_n^\ast}(t)$ and $\phi_U(t) \phi_V(t) = \phi_{U+V}(t)$, and therefore $U_n^\ast + V_n^\ast \to U + V$ weakly by Lévy's Continuity Theorem \citep[Theorem 2.13]{vandervaart2000asymptotic}. A similar argument shows that $U_n + V_n$ also converges weakly to the same limit, and hence the claimed result follows.

\subsubsection{Proof of \Cref{thm:mmd_bootstrap}}
\label{pf:thm:mmd_bootstrap}
Using a similar argument as in the proof in \Cref{pf:thm:ksd_bootstrap_test}, it suffices to show the claimed result with $\eta_{1-\alpha}^\infty$ replaced by the $(1-\alpha)$-th quantile, $\eta_{1-\alpha}$, of $S_{n,m}$, due to the bootstrap validity shown in \Cref{cor: validity of bootstrap}.  

We first prove the case \emph{(i)}, namely when $\MMD(Q, P) > \theta$. Using the symmetry and triangle inequality of MMD \citep[see, e.g.,][]{muller1997integral,gretton2012kernel} twice, we have
\begin{align}
    \MMD(Q, P)
    \;\leq\;
    \MMD(Q, Q_n) + \MMD(Q_n, P_m) + \MMD(P_m, P)
    \;=\;
    S_{n,m} + \MMD(Q_n, P_m)
    \label{eq:MMD_triangle_ineq}
    \;.
\end{align}
Using this inequality, we can bound the probability of rejecting $H_0$ as
\begin{align*}
    \Pr\big(T_{n,m}^{\MMD, \theta} > \eta_{1-\alpha}\big)
    \;&=\;
    \Pr\big(\theta - \MMD(Q_n, P_m) > \eta_{1-\alpha}\big)
    \\
    \;&=\;
    \Pr\big(\MMD(Q, P) - \MMD(Q_n, P_m) > \eta_{1-\alpha} + \MMD(Q, P) - \theta\big)
    \\
    \;&\leq\;
    \Pr\big(S_{n,m} > \eta_{1-\alpha} + \MMD(Q, P) - \theta \big)
    \\
    \;&=\;
    \Pr\big(\sqrt{N} S_{n,m} > \sqrt{N}\eta_{1-\alpha} + \sqrt{N}(\MMD(Q, P) - \theta) \big)
    \;.
    \tagaligneq
    \label{eq:mmd_test_rej_rewrite}
\end{align*}
We claim that $\sqrt{N} S_{n,m}$ converges weakly to a non-degenerate distribution. Indeed, by \Cref{assump:mmd_sample_sizes} and the definition of $S_{n,m}$, we can write $\sqrt{N}S_{n,m} = \nu_n^{-1/2}\sqrt{n}\MMD(Q_n, Q) + (1 - \nu_n)^{-1/2}\sqrt{m}\MMD(P_m, P)$, where $\nu_n \coloneqq n / N \to \nu$ by \Cref{assump:mmd_sample_sizes}, and both $\sqrt{n}\MMD(Q_n, Q)$ and $\sqrt{m}\MMD(P_m, P)$ converges weakly to some non-degenerate distributions by \Cref{lem:D_Qn_Q_asymptotics}. By the independence of $\MMD(P_n, P)$ and $\MMD(Q_m, Q)$, the statistic $\sqrt{N} S_{n,m}$ converges weakly to a non-degenerate distribution. In particular, its scaled quantile $\sqrt{N} \eta_{1-\alpha}$ also converges to some real number \citep[Lemma 21.2]{vaart1998asymptotic}. On the other hand, the term $\sqrt{N}(\MMD(Q, P) - \theta)$ diverges to $\infty$ when $\MMD(Q, P) > \theta$. Combining these arguments, we conclude that, when $\MMD(Q, P) > \theta$,
\begin{align*}
    \lim_{N \to \infty}
    \Pr\big(\sqrt{N} S_{n,m} > \sqrt{N}\eta_{1-\alpha} + \sqrt{N}(\MMD(Q, P) - \theta) \big)
    \;=\;
    0
    \;.
\end{align*}
We now consider the case \emph{(ii)}, namely $\MMD(Q, P) = \theta$. In particular, $Q \neq P$, and \Cref{prop:mmd_asymptotics} implies that the following weak limit holds
\begin{align*}
    \sqrt{N} \big( \MMD(Q_n, P_m) - \MMD(Q, P) \big) \;=\; \sqrt{N} \big( \MMD(Q_n, P_m) - \theta \big) \;\to\; \cN(0, \sigma^2_\MMD) \;.
\end{align*}
Moreover, since $\sqrt{N} \eta_{1-\alpha} \to c$ for some constant $c$ as argued before, we have
\begin{align*}
    \Pr\big(\theta - \MMD(Q_n, P_m) > \eta_{1-\alpha} \big)
    \;=\;
    \Pr\big(\sqrt{N}(\MMD(Q_n, P_m) - \theta) < -\sqrt{N}\eta_{1-\alpha} \big)
    \;\to\;
    \Phi\Big(-\frac{c}{\sigma_\MMD}\Big) 
    \;\eqqcolon\; 
    a \;,
\end{align*}
where $\Phi$ denotes the CDF of the standard normal distribution. In particular, we must have $a > 0$.
% Since $\sqrt{N} S_{n,m}$ converges weakly to a non-degenerate distribution, and since $\sqrt{N} \eta_{1-\alpha}$ also converges to some real number, there must exists a constant $a \in [0, 1]$ such that
% \begin{align*}
%     \Pr\big(\theta - \MMD(Q_n, P_m) > \eta_{1-\alpha} \big)
%     \;&=\;
%     \Pr\big(\sqrt{N} S_{n,m} > \sqrt{N}\eta_{1-\alpha} + \sqrt{N}(\MMD(Q, P) - \theta) \big)
%     \\
%     \;&=\;
%     \Pr\big(\sqrt{N} S_{n,m} > \sqrt{N}\eta_{1-\alpha} \big)
%     \\
%     \;&\to\;
%     a \;,
% \end{align*}
% where the first equality holds due to \eqref{eq:mmd_test_rej_rewrite}.
We now claim that $a \leq \alpha$. Applying the inequality \eqref{eq:MMD_triangle_ineq} again, we have the following bound on the probability of rejecting $H_0$
\begin{align*}
    \Pr\big(\theta - \MMD(Q_n, P_m) > \eta_{1-\alpha} \big)
    \;&\leq\;
    \Pr\big(\theta + S_{n,m} - \MMD(Q, P) > \eta_{1-\alpha} \big)
    \\
    \;&=\;
    \Pr\big(S_{n,m} > \eta_{1-\alpha} + \MMD(Q, P) - \theta \big)
    \\
    \;&=\;
    \Pr\big(S_{n,m} > \eta_{1-\alpha} \big)
    \\
    \;&=\;
    \alpha
    \;.
    % \label{eq:rej_prob_bound}
\end{align*}
Taking the limit $N \to \infty$ on both sides then implies $a \leq \alpha$. This shows \emph{(ii)}. To prove the case \emph{(iii)}, we write
\begin{align*}
    &\;
    \Pr\big(\theta - \MMD(Q_n, P_m) > \eta_{1-\alpha} \big)
    \\
    \;&=\;
    \Pr\big(\MMD(Q_n, P_m) < \theta - \eta_{1-\alpha} \big)
    \\
    \;&=\;
    \Pr\big(\MMD^2(Q_n, P_m) < (\theta - \eta_{1-\alpha})^2 \big)
    \tagaligneq
    \label{eq:mmd_power_step1}
    \\
    \;&=\;
    \Pr\big(\sqrt{N}\big(\MMD^2(Q_n, P_m) - \MMD^2(Q, P) \big) < \sqrt{N}(\theta - \eta_{1-\alpha})^2 - \sqrt{N}\MMD^2(Q, P) \big)
    \tagaligneq
    \label{eq:mmd_power}
    \;.
\end{align*}
Assuming first that $0 < \MMD(Q, P) < \theta$, then $\MMD^2(Q_n, P_m)$ is a non-degenerate V-statistic \citep[Proposition 1]{huang2023weighted}. Asymptotics of non-degenerate V-statistics \citep[Proposition 1 (i)]{huang2023weighted} then gives $\sqrt{N}\big(\MMD^2(Q_n, P_m) - \MMD^2(Q, P) \big) \to \cN(0, \sigma^2)$ in distribution, for some $\sigma^2 > 0$. On the other hand, the RHS of the inequality within the probability in \eqref{eq:mmd_power} can be simplified as
\begin{align*}
    \sqrt{N}(\theta - \eta_{1-\alpha})^2 - \sqrt{N}\MMD^2(Q, P)
    \;=\;
    \sqrt{N}(\theta^2 - \MMD^2(Q, P)) + \sqrt{N} \eta_{1-\alpha}^2 - 2 \sqrt{N}\theta \eta_{1-\alpha}
    \;.
\end{align*}
By \Cref{lem:D_Qn_Q_asymptotics}, the scaled quantile $\sqrt{N} \theta \eta_{1-\alpha}$ converges to a positive constant, while $\sqrt{N} \eta_{1 - \alpha}^2 \to 0$. Since also $\MMD(Q, P) < \theta$, the RHS of the above equality diverges to $\infty$. We have therefore showed that, inside the probability of \eqref{eq:mmd_power}, the LHS converges weakly to a Gaussian limit, while the RHS diverges to $\infty$. This implies that $\Pr\big(\theta - \MMD(Q_n, P_m) > \eta_{1-\alpha} \big) \to 1$.

It then remains to prove \emph{(iii)} when $\MMD(Q, P) = 0$. In this case, the V-statistics $\MMD^2(Q_n, P_n)$ is \emph{degenerate of order 2} \citep[see, e.g., ][Section 2]{huang2023weighted}, and the asymptotics of 2-degenerate V-statistics \citep[Proposition 1 (ii)]{huang2023weighted} imply that $N \MMD^2(Q_n, P_n)$ converges weakly to a non-degenerate distribution. In particular, $\MMD^2(Q_n, P_m) \to 0$ in probability. On the other hand, $\eta_{1-\alpha} \to 0$ in probability by \Cref{lem:D_Qn_Q_asymptotics}, and so $(\theta - \eta_{1-\alpha})^2 \to \theta > 0$. In other words, inside the probability in \eqref{eq:mmd_power_step1}, the LHS converges to $0$ in probability, while the RHS converges to a non-zero constant. This implies that \eqref{eq:mmd_power_step1} converges to $1$, thus showing the case \emph{(iii)} when $\MMD(Q, P) = 0$.

\subsubsection{Proof of \Cref{thm:power_selected_theta_ksd}}
First assume $\KSD(Q, P) \leq \theta'$. Inequality \eqref{eq:KSD_traingle_ineq} implies
\begin{align}
    T_{n,m}^{\KSD, \theta}
    \;=\;
    \theta - \KSD(Q_n, P)
    \;&\geq\;
    \theta - \MMD(Q_n, Q; u_p) - \KSD(Q, P)
    \;\geq\;
    \theta - \theta' - \xi_n
    \label{eq:power_triangle_ineq_ksd}
    \;,
\end{align}
where the last inequality holds since $\KSD(Q, P) \leq \theta'$, and we have defined $\xi_{n} \coloneqq \MMD(Q_n, Q; u_p)$. Therefore,
\begin{align*}
    \Pr(T_{n}^{\KSD,\theta} \;>\; \eta_{1-\alpha})
    \;\geq\;
    \Pr(\theta - \theta' - \xi_n > \eta_{1-\alpha})
    \;\stackrel{(i)}{=}\;
    \Pr(\xi_n < \eta_{1-\beta})
    \;\stackrel{(ii)}{=}\;
    1 - \beta
    \;,
\end{align*}
where \emph{(i)} holds since $\theta = \theta + \eta_{1-\alpha} + \eta_{1-\beta}$ by definition, and \emph{(ii)} holds as $\eta_{1-\beta}$ is the $(1-\beta)$-th quantile of $\xi_n$. This shows the first claim.

Now assume $\KSD(Q, P) > \theta'$. The probability of rejection can be bounded as
\begin{align*}
    \Pr(T_{n}^{\KSD,\theta} \;>\; \eta_{1-\alpha})
    \;&=\;
    \Pr(\theta - \KSD(Q_n, P) \;>\; \eta_{1-\alpha})
    \\
    \;&=\;
    \Pr(\theta' + \eta_{1-\alpha} + \eta_{1-\beta} - \KSD(Q_n, P) \;>\; \eta_{1-\alpha})
    \\
    \;&=\;
    \Pr(\KSD(Q_n, P) \;<\; \theta' - \eta_{1-\beta})
    \tagaligneq \label{eq:varying_theta_rej_prob_bound}
    \;.
\end{align*}
The statistic $\xi_n = \MMD(Q_n, Q; u_p) \to \MMD(Q, Q; u_p) = 0$ almost surely by the strong law of large numbers for V-statistics \citep{kokic1987rates}. Hence, its quantile $\eta_{1-\beta}$ also converges to 0. On the other hand,
\begin{align*}
    \KSD(Q_n, P) \;\to\; \KSD(Q, P) \;>\; \theta' \;,
\end{align*}
where the convergence is again by the law of large numbers. Combining these argument shows that the RHS of \eqref{eq:varying_theta_rej_prob_bound} converges to 0, thus showing the claimed result.

\subsubsection{Proof of \Cref{thm:power_selected_theta}}
\label{pf:thm:power_selected_theta}
An inequality similar to \eqref{eq:power_triangle_ineq_ksd} can be derived for MMD as follows
\begin{align*}
    T_{n,m}^{\MMD, \theta}
    \;=\;
    \theta - \MMD(Q_n, P_m)
    \;&\geq\;
    \theta - \MMD(Q_n, Q) - \MMD(Q, P) - \MMD(P_m, P)
    \;\geq\;
    \theta - \theta' - S_{n,m}
    \;,
\end{align*}
where the last inequality holds since $\MMD(Q, P) \leq \theta'$, and $S_{n,m} \coloneqq \MMD(Q_n, Q) + \MMD(P_m, P)$. The rest of the proof is the same as that of \Cref{thm:power_selected_theta} by replacing \eqref{eq:power_triangle_ineq_ksd} with the above inequality, and using the quantiles of $S_{n,m}$ instead of $\xi_n$.

\section{Computation of the MMD Asymptotic Variance}
\label{app:mmd_normal_vars}
We show how \eqref{eq:mmd_normal_vars} can be computed in $\cO((n + m)^2)$ time. We provide a discussion for $\hat{\sigma}_{\MMD, 1}^2$, and the same argument applies to $\hat{\sigma}_{\MMD, 2}^2$. 

The key observation is that $q_i$, defined below \eqref{eq:mmd_normal_vars}, can be simplified to double sums using the definition of the MMD kernel $h(x, x', y, y') = k(x, x') + k(y, y') - k(x, y') - k(x', y)$. Indeed,
\begin{align*}
    (n-1)m(m-1)q_i \;&=\; \sum_{\substack{i' = 1\\ i' \neq i}}^n \sum_{j=1}^m \sum_{\substack{j'=1\\ j' \neq j}}^m h(x_i, x_{i'}, y_j, y_{j'})
    \\
    \;&=\;
    \sum_{\substack{i' = 1\\ i' \neq i}}^n \sum_{j=1}^m \sum_{\substack{j'=1\\ j' \neq j}}^m k(x_i, x_{i'}) + k(y_j, y_{j'}) - k(x_i, y_{j'}) - k(x_{i'}, y_j)
    \\
    \;&=\;
    m(m-1)\sum_{\substack{i' = 1\\ i' \neq i}}^n k(x_i, x_{i'}) 
    + (n-1) \sum_{j=1}^m \sum_{\substack{j'=1\\ j' \neq j}}^m k(y_j, y_{j'})
    - (n-1)(m-1) \sum_{j'=1}^m k(x_i, y_{j'})
    \\
    \;&\quad
    - (m-1) \sum_{\substack{i' = 1\\ i' \neq i}}^n \sum_{j=1}^m k(x_{i'}, y_j)
    \;.
\end{align*}
Computing the first and third terms for all $\{q_i\}_{i=1}^n$ requires $\cO(n^2 + nm)$. The second term requires $\cO(m^2)$ and only needs to be computed once for all $q_i$. For the last term, we can rewrite it as
\begin{align*}
    \sum_{\substack{i' = 1\\ i' \neq i}}^n \sum_{j=1}^m k(x_{i'}, y_j)
    \;=\;
    \sum_{i' = 1}^n \sum_{j=1}^m k(x_{i'}, y_j) - \sum_{j=1}^m k(x_i, y_j)
    \;=\;
    T_{1} + T_{i,2}
    \;.
\end{align*}
Computing $T_{i,2}$ for all $q_i$ requires $\cO(nm)$, while $T_{1}$ only needs to be computed once, which also costs $\cO(nm)$. Therefore, all $\{q_i\}_{i=1}^n$ can be computed in $\cO((n + m)^2)$ time. Since $\hat{\sigma}_{\MMD, 1}^2$ is the sample variance of $\{q_i\}_{i=1}^n$, it can also be computed in $\cO((n + m)^2)$ time.

\end{document}

%% file: preamble.tex
\usepackage{arxiv}

\usepackage[utf8]{inputenc} % allow utf-8 input
\usepackage[T1]{fontenc}    % use 8-bit T1 fonts
\usepackage{hyperref}       % hyperlinks
\usepackage{url}            % simple URL typesetting
\usepackage{booktabs}       % professional-quality tables
\usepackage{amsfonts}       % blackboard math symbols
\usepackage{nicefrac}       % compact symbols for 1/2, etc.
\usepackage{microtype}      % microtypography
\usepackage{xcolor}         % colors
\usepackage{lipsum}		    % Can be removed after putting your text content
\usepackage[round]{natbib}

\usepackage{amsmath}
\usepackage{amssymb}
\usepackage{amsthm}
\usepackage{mathtools, nccmath}
\usepackage{amsfonts}
\usepackage{bm}
\usepackage{graphicx}
\usepackage{algorithm}
\usepackage{algpseudocode}
\usepackage{dsfont}
\usepackage{caption}
\usepackage{subcaption}
\usepackage{centernot}
\usepackage{wrapfig}
\usepackage{adjustbox}
\usepackage{bbm}
\usepackage{cleveref}
\hypersetup{
    colorlinks=true,
    linkcolor=blue,
    citecolor=blue,      
    urlcolor=brown,
}

\newtheorem{theorem}{Theorem}

\newtheorem{lemma}[theorem]{Lemma}
\newtheorem{proposition}[theorem]{Proposition}
\newtheorem{assumption}{Assumption}
\newtheorem{remark}{Remark}[theorem]

\newcommand{\tagaligneq}{\refstepcounter{equation}\tag{\theequation}}

\def\E{\mathbb{E}}

\def\R{\mathbb{R}}

\def\X{\mathbb{X}}
\def\Y{\mathbb{Y}}
\def\Z{\mathbb{Z}}

\def\cA{\mathcal{A}}

\def\cC{\mathcal{C}}

\def\cH{\mathcal{H}}

\def\cN{\mathcal{N}}
\def\cO{\mathcal{O}}
\def\cP{\mathcal{P}}

\newcommand{\Var}{\text{\rm Var}}

\newcommand{\darrow}{\xrightarrow{\;\text{\tiny\rm d}\;}}

\newcommand*\diff{\mathop{}\!\mathrm{d}}
\newcommand{\indicator}{\mathbbm{1}}
\newcommand{\MMD}{\mathrm{MMD}}
\newcommand{\KSD}{\mathrm{KSD}}